\pdfoutput=1
\documentclass[11pt]{article}
\fontsize{10}{5}

\usepackage{amsthm, amsmath, amssymb}
\usepackage{graphicx}
\usepackage[tight]{subfigure}
\newtheorem{theorem}{Theorem}[section]

\DeclareMathOperator{\MMD}{MMD}

\DeclareMathOperator{\HS}{HS}

\newcommand{\al}{\alpha}

\newcommand{\e}{\epsilon}
\newcommand{\la}{\lambda}
\newcommand{\s}{\sigma}

\newcommand{\R}{\mathbb{R}}

\newcommand{\N}{\mathbb{N}}

\newcommand{\G}{\mathcal{G}}

\newcommand{\La}{\mathcal{L}}
\newcommand{\W}{\mathcal{W}}
\newcommand{\X}{\mathcal{X}}

\newcommand{\vp}[2]{\varphi\left(\frac{\|#1\|^p}{#2^p}\right)}

\newcommand{\cq}[1]{c_{\al(#1)}}

\def\i{\infty}
\def\l{\langle}
\def\v{\varphi}
\def\r{\rangle}
\def\m{\mathbf}
\def\H{\mathcal{H}}
\def\O{\mathcal{O}}

\newcommand{\sampSet}{\mathcal{X}}
\newcommand{\sampSetOther}{\mathcal{Y}}
\newcommand{\sampSetLong}{\{x_1, \dots, x_{\nsamp}\}}
\newcommand{\sampSetShort}{\{x_i\}_{1}^{\nsamp}}
\newcommand{\idealP}{\ensuremath{{p}}}
\newcommand{\idealK}{\ensuremath{{\m{K}}}}

\newcommand{\empP}{\ensuremath{\widehat p}}
\newcommand{\empK}{\ensuremath{K}}
\newcommand{\shCent}{\mathcal{C}}
\newcommand{\shCentLong}{\{c_1,\dots,c_{\ncent}\}}
\newcommand{\weightSet}{\W}
\newcommand{\weightSetLong}{\{w_1, \dots, w_{\ncent}\}}

\newcommand{\shP}{\ensuremath{\widetilde p}}
\newcommand{\shK}{\ensuremath{\widetilde K}}
\newcommand{\idealshK}{\ensuremath{\widetilde{\m{K}}}}
\newcommand{\shG}{\ensuremath{\widetilde G}}
\newcommand{\CK}{K^{\shCent}}
\newcommand{\quaK}{\ensuremath{\oline{K}}}
\newcommand{\quaC}{\ensuremath{\oline{\mathcal{C}}}}
\newcommand{\card}[1]{\left|{#1}\right|}
\newcommand{\nsamp}{n}
\newcommand{\ncent}{m}
\newcommand{\ntrain}{n_{t}}

\newcommand{\neigs}{r}

\newcommand{\fmap}{\ensuremath{\psi}}
\newcommand{\shSet}{S}
\newcommand{\shrad}{\varepsilon}

\newcommand{\eqlabel}[1]{\label{eq:#1}}

\renewcommand{\eqref}[1]{(\ref{eq:#1})}

\makeatletter
\renewcommand{\paragraph}{%
  \@startsection{paragraph}{4}%
  {\z@}{1.5ex \@plus 0.5ex \@minus .2ex}{-1em}%
  {\normalfont\normalsize\bfseries}%
}
\makeatother

\newcommand*\oline[1]{%
  \vbox{%
    \hrule height 0.25pt%                  % Line above with certain width
    \kern 0.25ex%                          % Distance between line and content
    \hbox{%
      \kern-0.2em%                        % Distance between content and left side of box, negative values for lines shorter than content
      \ifmmode#1\else\ensuremath{#1}\fi%  % The content, typeset in dependence of mode
      \kern-0.2em%                        % Distance between content and left side of box, negative values for lines shorter than content
    }% end of hbox
  }% end of vbox
}

\usepackage{algorithm}
\usepackage{algorithmic}

\title{Reduced-Set Kernel Principal Components Analysis for Improving the Training and Execution Speed
  of Kernel Machines}
  \author{Hassan A. Kingravi
  \thanks{School of Electrical and Computer Engineering, Georgia Institute of Technology.}
  \and Patricio A. Vela\footnotemark[1] \and 
  Alexandar Gray \thanks{College of Computing, Georgia Institute of Technology.} 
  }
  \date{}

\begin{document}

\maketitle

\begin{abstract}
%\small\baselineskip=9pt
This paper \footnote{This paper first appeared in \emph{SIAM International Conference on Data Mining, 2013.}}
presents a practical, and theoretically well-founded,
approach to improve the speed of kernel manifold learning algorithms
relying on spectral decomposition.
% \footnote{}.
Utilizing recent insights in
kernel smoothing and learning with integral operators, we propose
Reduced Set KPCA (RSKPCA), which also suggests an easy-to-implement method
to remove or replace samples with minimal effect on the empirical
operator. A simple data point selection procedure is given to generate
a substitute density for the data, with accuracy that is governed by a
user-tunable parameter $\ell$.  The effect of the approximation on
the quality of the KPCA solution, in terms of spectral and operator
errors, can be shown directly in terms of the density estimate error
and as a function of the parameter $\ell$. We show in experiments that
RSKPCA can improve both training and evaluation time of KPCA by up to
an order of magnitude, and compares favorably to the widely-used
Nystr\"{o}m and density-weighted Nystr\"{o}m methods.
\end{abstract}

\section{Introduction}
\label{submission}
 
Modern problems in machine learning are characterized by large, often
redundant, high-dimensional datasets.
To interpret and more effectively use high-dimensional data, a simplifying
assumption often made is that the data lies on an embedded manifold.
Recovery of the underlying manifold aids certain machine learning problems
such as deriving a classifier from the data, or estimating a function
of interest.
Algorithms that try to recover this underlying structure within the field of
manifold learning include methods such as
% Locally Linear Embedding \cite{Saul},
Laplacian eigenmaps \cite{Belkin} and
diffusion maps \cite{Coifman:2006}.
Many such methods can be thought of as Kernel PCA (KPCA) \cite{Scholkopf1998}
performed on specially constructed kernel matrices \cite{Ham}.  We denote this
class of methods as Kernel Manifold Learning Algorithms.
%While providing good empirical results for problems such as image denoising \cite{Scholkopf1998},
%face recognition \cite{He}, and novelty detection \cite{Hoffman}, the fact
For a dataset with $\nsamp$ points, KMLAs involve the eigendecomposition of
an $\nsamp\times \nsamp$ kernel matrix $K$, and a manifold mapping of order
$\O(\nsamp)$ in cost (for a dataset with $\nsamp$ points), which limits
their usefulness in some application domains (e.g., online learning and
visual tracking).  In addition to the computational cost, storage of the
kernel matrix in memory becomes difficult for larger datasets, particularly
for kernels such as the Gaussian which tends to generate dense matrices.
Therefore a truly scalable KMLA method should be
one that 1) avoids the computation of the full kernel matrix,  2) has low
training cost, and 3) has low testing cost.

Existing methods for speeding up the computation time of KMLAs focus on the
training and testing phases separately. Examples of the former include methods
such as Incomplete Cholesky Decomposition (ICD) \cite{Shawe-Taylor:2004},
the Nystr\"{o}m method \cite{Drineas} and random projections\cite{Achlioptas:2002},
which compute a low rank approximation of the kernel matrix in terms of the
original dataset with $\nsamp$ points and a subset of $\ncent$ points
(see \cite{Zhang} and the references therein).  While exhibiting excellent
performance, ICD, random projections and certain Nystr\"{o}m methods require
the computation of the kernel matrix.
An example of a Nystr\"{o}m method that does not require the computation of
a kernel matrix is one where the centers are chosen uniformly from the data.
While performing well in practice, it suffers from the lack of a principled
way to choose the number of centers. Related work in the class is
\cite{Zhang}, which employs $k$-means clustering and a density-weighted Gram
matrix for performing KPCA.  Drawbacks to the approach include the use of
$k$-means, which also requires the number of clusters in advance and can be
slow in high dimensions (due to its iterative nature); and an asymmetric
weighted Gram matrix.  Further, both methods require the retention of the
full dataset for computing projections; while the training cost may be
lower, the testing cost remains the same.

Methods to reduce the testing cost include reduced set selection and sparse
selection methods, which find a reduced set of expansion vectors from
the original space that approximate well the training set
\cite{Scholkopf:1999,Tipping:2000}, reduced set construction, which
identifies new elements of the input space that approximate well the
training set \cite{Scholkopf:1999}, and kernel
map compression, which uses generalized radial basis function networks to
approximate the kernel map \cite{Arif:2011}. Given that the full
eigendecomposition is typically required, these methods tend to be expensive
in training, but can reduce the testing cost significantly.

{\bf Approach.} To the authors' knowledge, no method exists which considers
speeding up both training and testing of KMLAs in a unified and principled manner.
This paper proposes to do so by connecting kernel principal component analysis
to the eigendecomposition of kernel smoothing operators.  In particular,
% The spectral content of the operators associated to KMLAs is determined by
% the underlying probability density $\idealP(x)$
%  \cite{Williams,Shawe-taylor:2005,Rosasco:2010}.
given a sampled data set $\{x_i\}_1^{\nsamp}$, we show that the spectral
decomposition of the Gram matrix $K$ is related to the kernel density estimate
$\empP(x)$.  If an approximation $\shP(x)$ is available whose cardinality is
much lower than that of $\empP(x)$, an approximation to the original Gram
matrix can be computed at a significantly reduced computational cost, thus
improving the execution of KMLAs.

{\bf Contribution.}
There are two main contributions in the paper.
This paper first exploits the connection of kernel smoothing to the spectral
decomposition of integral operators, within the context of kernel
principal component analysis (KPCA), to define reduced set KPCA
(RSKCPA).  RSKPCA relies on the existence of a reduced set density estimate
(RSDE) of the dataset, with a cardinality of $\ncent$ rather than $\nsamp$
(where $\ncent \ll \nsamp$).
The RSDE defines a weighted $m\times m$ Gram matrix $\shK$, whose
eigendecomposition is computed in lieu of the empirical Gram matrix $\empK$.
The RSKPCA approach circumvents the computation of the full kernel matrix so
that the eigendecomposition is of order $\O({\ncent}^3)$ cost instead
of $\O({\nsamp}^3)$.  Evaluation time is also reduced, as mapping a test
point into the reduced eigenspace requires $O(k\ncent)$ operations rather
than $O(k\nsamp)$, with $k$ retained eigenvectors.

%, which exploits the quantization effect induced in the kernel matrix by the
%kernel.
While many methods can be used to generate the reduced set approximation
$\shP(x)$ to the empirical density $\empP(x)$, efficient methods are
preferred in order to truly impact the overall training time.  This
paper proposes a simple, fast, single-pass method relying on the concept of
the `shadow' of a radially-symmetric kernel to generate the approximation
$\shP(x)$, called the shadow density estimate (ShDE).  The ShDE depends on a
user-tuned parameter $\ell$ to arrive at an RSDE of cardinality
$\ncent \ll \nsamp$ with a run-time cost of $\O({\ncent \nsamp})$.
Unlike previous work where $m$ is chosen arbitrarily, $\ell$ is related to
the kernel, and can generally be set to a generic value (say $\ell = 4$) for
a wide variety of problems.

The shadow algorithm enables the derivation of closed form error bounds of
the RSDE and RSKPCA results.  Results bounding {\bf (1)} the approximation
of the density via the Maximum Mean Discrepancy (MMD), {\bf (2)} the eigenvalue difference
between the operators $\empK$ and $\shK$, and {\bf (3)} the difference in
Hilbert-Schmidt norm between the operators and their eigenspace projections,
provide further theoretical justification for the approach.  The bounds are
given in terms of the user-tuned parameter $\ell$.  The latter two bounds
are shown to be directly related to the first bound, indicating the
importance of the density estimate in generating a correct
eigendecomposition.  The proposed approach performs well as a substitute for
the Nystr\"{o}m family of algorithms.  While the application of choice in this
paper is KMLAs, the method is applicable any problem which satisfies the
assumptions and which can be formulated as a kernel eigenvalue problem.

{\bf Organization.}
Section \ref{sec:kpca} reviews the operator view of KPCA.  Theoretical support
for reduced set KPCA (RSKPCA) follows in Section \ref{sec:rskpca}, which uses
the connection to kernel smoothing to define RSKPCA.  Section
\ref{sec:shadow} defines the shadow of the kernel from which the shadow
density estimate (ShDE) is derived and used in the RSKPCA
algorithm.   Section \ref{sec:bounds} provides error bounds on the MMD distance
between the KDE and the ShDE, and the approximation of the operator by RSKPCA.
Section \ref{sec:experiments} reports experimental results, which show the efficacy of the
method on speeding up KPCA and KPCA-based methods.

\section{KPCA and Eigenfunction Learning}\label{sec:kpca}
This section briefly summarizes the foundations of KPCA as regards the
spectral decomposition of operators. To start,
let $k:\R^d\times\R^d\to\R$ be a bounded, positive-definite kernel function,
defined on the domain $D\subset \R^d$. Then $k$ has the property
$k(x,y) = \l\fmap(x),\fmap(y)\r_{\H}$ where $\H$ is a Reproducing Kernel
Hilbert space and $\fmap:\R^d\to\H$ is an implicit mapping.
The kernel induces a linear operator $\idealK:\La^2(D)\to \La^2(D)$,
%defined by
\begin{align} \eqlabel{opdef1}
  (\idealK f)(x) := \int_D k(x,y)f(y)dy.
\end{align}
To incorporate data arising from a probability density $\idealP(x)$,
\eqref{opdef1} can be modified.  Let $\mu$ be a probability measure on
$D$ associated to $\idealP$, and denote by $\La^2(D,\mu)$ the space of
square integrable functions with norm $\|f\|_p^2
= \l f,f\r_p = \int_D f(x)^2d\mu(x)$.
Define the linear operator  $\idealshK:\La^2(D,\mu)\to \La^2(D,\mu)$ by
\begin{align} \eqlabel{opdef2}
 (\idealshK f)(x) := %\int_D k(x,y)f(y)d\mu(y)\\ =
                 \int_D k(x,y)f(y)\idealP(y)dy.
\end{align}
The operator $\idealshK$ is associated to the eigenproblem
\begin{align} \eqlabel{op2}
 \int_D k(x,y)\idealP(x)\phi_\iota(x)dx = \la_\iota \phi_\iota(y),
\end{align}
where $\phi_\iota(\cdot)$ are the eigenfunctions.
In practice, given a sample set $\sampSet = \sampSetShort$ drawn from
$\idealP(x)$, the empirical approximation to \eqref{op2} is derived from
the approximation
\begin{align} \eqlabel{emp_op1}
 \int_D k(x,y)\idealP(x)\phi_\iota(x)dx \approx \frac{1}{\nsamp}\sum_{i=1}^\nsamp k(x_i,y)\phi_\iota(x_i),
\end{align}
as obtained from the empirical estimate of the probability density
$\idealP(x)$ using $\sampSet$,
\begin{align} \eqlabel{del_p}
 \idealP(x) \approx \frac{1}{\nsamp}\sum_{i=1}^\nsamp \delta(x_i,x),
\end{align}
which employs the sampling property of the delta function.
Equation \eqref{emp_op1} then leads to the eigendecomposition of the Gram
matrix $\empK$
\begin{align} \eqlabel{emp_op2}
 \empK \hat{\phi}_i = \hat{\la}_i \hat{\phi}_i,
  \quad \empK_{ij}:= k(x_i,x_j)\
\end{align}
for $x_i,\  x_j \in \sampSet$, where ($\hat{\la}_i,\ \hat{\phi}_i$) are the
eigenvalue and eigenvector pairs of $K$ in the finite-dimensional
subspace generated by the mapped data points, $x_i \mapsto k(x_i, \cdot)$.
Kernel principal component analysis (KPCA) further scales the eigenvectors
of $\empK$ by their eigenvalues to achieve orthonormality.  As the number of
samples $\nsamp\to\i$, the approximation converges to the true eigenvalues
and eigenfunctions of \eqref{op2} \cite{Williams,Bengio}.

% The approach relies on the eigendecomposition
% of \eqref{op3} and its connection to the density.

\section{Reduced Set KPCA}
\label{sec:rskpca}

This section proposes an alternative formulation of the operator and its
spectral decomposition in order to derive reduced set KPCA, as based on an
approximation to the empirically determined kernel density estimate.
First, note that the integral equation leading to KPCA, Eq. \eqref{opdef2},
implies a kernel \emph{smoothing} of the density (using the operator $\idealK$
applied to $\idealP$),
\begin{equation} \eqlabel{k_smooth}
 (\idealK \idealP)(x) = \int k(x,y)\idealP(y)dy.
\end{equation}
Given a set of samples $\sampSet = \sampSetLong$ drawn from the density
$\idealP$ and using \eqref{del_p}, the smoothed approximation \eqref{k_smooth}
is obtained as
\begin{equation} \eqlabel{k_discrete}
 \empP(x) = (\idealK \idealP)(x)
          \approx \frac{1}{\nsamp}\sum_{i=1}^{\nsamp} k(x_i,x),
\end{equation}
which is known as the kernel density estimate (KDE) \cite{Wasserman}.
The KDE converges to $\idealP(x)$ under some mild assumptions, however using
it can be expensive due to the $\O(\nsamp)$ operations required to compute
$\empP(x)$, thus it is common to utilize a reduced set density estimate
\begin{equation} \eqlabel{approxprob}
 \shP(x) = \frac{1}{\nsamp}\sum_{i=1}^{\ncent} w_i k(c_i, x),
\end{equation}
where $\W = \weightSetLong$, $\shCent = \shCentLong$, and $m<<n$.  The
empirical density generating $\shP$ under the kernel smoother $\idealK$ is
\begin{align} \eqlabel{del_p2}
 \idealP(x) \approx \frac{1}{\ncent}\sum_{i=1}^\ncent w_i \delta(c_i,x).
\end{align}
While having quite different generating approximations, the kernel smoothed
density $\shP$ is close to $\empP$ by construction
\cite{Chen:2010,Freedman,Zhang}.
This paper will replace the KPCA procedure of the eigenproblem derived from
\eqref{emp_op2} and {\eqref{del_p} with one derived from \eqref{approxprob}
and \eqref{del_p2} using an alternative, equivalent formulation of the
continuous eigenproblem \eqref{op2}.  The formulation considers the kernel
\begin{align} \eqlabel{modkernel}
 \tilde{k}(x,y) = p^{1/2}(x)k(x,y)p^{1/2}(y),
\end{align}
which is a \emph{density weighted} version of the original kernel.
%This notion of density weighting corresponds to a larger concentration of
%samples in the regions of the domain where $\idealP(x)$ is high.
The eigenvalues of \eqref{modkernel} are the same as those of
\eqref{op2} \cite{Williams}.  Therefore, the eigenproblem of \eqref{op2} is
the same as the eigenproblem
\begin{align} \eqlabel{op3}
 \int \tilde{k}(x,y)\tilde{\phi}_\iota(x)dx = \la_\iota\tilde{\phi}_\iota(y),
\end{align}
where the relationship between the two eigenvector sets is that
$\tilde{\phi}_\iota(\cdot) = p^{1/2}(\cdot)\phi_\iota(\cdot)$.
Using \eqref{del_p2} and \eqref{modkernel} in \eqref{op3} gives an
eigendecomposition problem with the reduced set Gram matrix
\begin{align} \eqlabel{modkernel_emp}
  \shK \tilde{\phi}_i = \tilde \lambda_i \tilde{\phi}_i,
  \quad \shK_{ij} := \sqrt{w_i}k(c_i,c_j) \sqrt{w_j},
\end{align}
for $c_i,\,c_j \in \shCent$.  The proposed reduced set KPCA procedure
replaces the  Gram matrix $\empK$ in the empirical eigenproblem
\eqref{emp_op2} by a density weighted surrogate
\begin{align*}
 \shK = W\CK W^T,
\end{align*}
where $\CK_{ij} := k(c_i,c_j)$, $W = diag(\sqrt{w_1}, \dots, \sqrt{w_m})$ is
the weight matrix.  The matrix $\shK$ is an empirical, finite-dimensional
approximation to \eqref{modkernel}.  Unlike $\empK$, $\CK$  is an
$\ncent\times\ncent$ matrix (as is $\shK$). Once the centers are selected
and the weights computed using a reduced set density estimation algorithm,
\emph{the original data is discarded}. This makes the algorithm
fundamentally different from  Nystr\"{o}m type methods which retain the
training data for eigenfunction computations at test time, and both the
sparse approximation and the eigenvector approximation methods which need to
first compute the eigendecomposition of a full kernel matrix to generate the
reduced set eigenfunction computations for testing.  The algorithm can be
more aggressive with the training data than either of these two strategies
in pursuit of both training and testing speedups.  The reduced set KPCA
algorithm is summarized in Algorithm \ref{alg:shkpca}.  Since the full
kernel matrix is never computed once an RSDE is available, the training cost
of the algorithm is $\O(m^3)$ and the testing cost is $\O(m)$.

\begin{figure}[t!]
\centering
 \begin{algorithm}[H]
\caption{Reduced Set KPCA}
   \label{alg:shkpca}
\begin{algorithmic}
\begin{footnotesize}
  \STATE Apply a reduced set density estimator to $\sampSet$ to compute  \\
     \quad $\shCent = \{c_1,\dots,c_m\}$ and $w = \{w_1,\dots,w_m\}$.
  \STATE Create diagonal matrix $W = \text{diag}(\sqrt{w_1}, \dots, \sqrt{w_m})$.
  \STATE Compute weighted kernel matrix
       \begin{equation*}
        \shK \in \R^{m\times m}, \quad \shK := W\CK W
       \end{equation*}
       \quad where $\CK_{ij} := k(c_i,c_j)$.
   \STATE Perform eigenvector decomposition
     $\shK\tilde \phi_i = \la_i \tilde \phi_i$
   \STATE Reweight to get the eigenvectors
     $\hat \phi_i = W^{-1/2} \tilde \phi_i$.
\end{footnotesize}
\end{algorithmic}
\end{algorithm}
\end{figure}

The key insight into the procedure is that an accurate reduced set density
estimate must lead to a similarly accurate reduced set KPCA.  This is seen
by noting that the KDE and the RSDE both arise as empirical approximations
to the same continuous eigenproblem.

{\bf Extension to KMLAs.}
More generally, there is a class of manifold learning methods that can be
reformulated as the following generic eigenproblem
\begin{align}
 \label{generic_eigen}
 (\G f)(x) = \int_D g(x,y)k(x,y)f(y)p(y)dy.
\end{align}
If $\G$ is a positive definite operator, it generates an RKHS $\H$.
An equivalent eigenproblem is of the form
\begin{align}
 \label{generic_approx}
 (\shG f)(x)  = \int_D g(x,y)f(y)\tilde{k}(x,y)dy.
\end{align}
Given algorithms where the integral operator is of the form (\ref{generic_approx})
(such as diffusion maps, Laplacian eigenmaps, normalized cut etc), approximation algorithms similar to
Algorithm \ref{alg:shkpca} can be formulated.

\section{A Fast and Simple RSDE}
\label{sec:shadow}

Here, a specific RSDE algorithm for use within RSKPCA, to improve the
execution time of learning and testing versus KPCA, is given.  By proposing a
simple algorithm, closed form approximation errors are computable as
explored in the subsequent section.

While many algorithms have been
designed for reduced set density estimation, to meet our purposes, the
RSDE must satisfy three criteria:
1) it must incorporate the kernel within its estimate;
2) its computational cost cannot be excessive, as that would fail to speed up
the KMLA; and
3) the number of centers $\ncent$ must be identified in a principled way,
since they may vary from problem to problem, and must have deterministic
approximation error.  These three criteria are met by a simple algorithm
exploiting the structure of radially symmetric kernels.
An approach similar
to the one proposed here is found in \cite{Wang}, however their
selection parameter is not fundamentally related to the kernel bandwidth and
they draw no connection to KPCA.

Given a bounded kernel function
$k(\cdot,\cdot)$, where $\kappa$ is the maximum value attained at
$k(c,c),\ \forall c\in\R^d$, and a
sequence $\{y_i\}_{i\in\N}$, if $\|c-y_i\|\to 0$, then
$k(c,y_i)\to \kappa$ (as $i \to \i$). Points sufficiently close to $c$
seem indistinguishable from the perspective of the kernel centered at $c$.
Declare such points near $c$ to lie in the shadow of the kernel function at
$c$.  Given a dataset $\sampSetShort$ used to determine $\empP(x)$, all
points of the dataset in the shadow of another point $c \in \sampSetShort$
can be replaced with $c$ at minor cost.  Removing the now duplicate points
requires an increase in the weight of $c$ by the number of points removed in
the KDE.  Extending this idea further, suppose that there existed a
collection of points from $\sampSetShort$ whose $\shrad$-balls covered the
entire dataset (with $\shrad$ to be defined shortly), then points lying in
these $\shrad$-balls could be removed with minor effect, leading to the
shadow density estimate:
\begin{equation} \eqlabel{shP}
 \shP(x) := \frac{1}{n}\sum_{j=1}^{\ncent}w_j k(c_j,x)
         \approx \frac{1}{n} \sum_{j=1}^{\ncent}\sum_{\xi\in \shSet_j}k(\xi,x)
\end{equation}
where $\shSet_j$ is the set of points lying in the shadow of the
point $c_j$, $w_j = |\shSet_j|$, and $\shSet_i \cap \shSet_j = \emptyset$
when $i \neq j$.  This paper specializes to the case of radially symmetric
kernels with bandwidth parameter $\sigma$, and defines $\shrad$ to be
determined by a parameter $\ell$ via $\shrad(\ell) = \sigma/\ell$.
What remains is to provide a selection procedure for the shadow centers $c_j$.
Algorithm \ref{alg:shadow} provides a single-pass $\O(mn)$ complexity
approach\footnote{The parameter $\ell$ implicitly determines the number
$\ncent$.}.
Figure \ref{shVis} conceptually depicts the process of moving from data to
shadow centers, and also the reconstruction of the KDE using a ShKDE.  The
color coding depicts the distinct shadow sets.  Based on \S\ref{sec:kpca},
the RSKPCA procedure follows as in Algorithm \ref{alg:shkpca}.  The next
section utilizes $\epsilon(\ell)$ to analyze the effectiveness of the ShDE
approximation and the fidelity of RSKPCA. The experiments section discusses
other RSDEs, and compares ShDE to them in the context of RSKPCA.

\begin{figure}[t!]
\centering
\begin{algorithm}[H]
   \caption{Shadow Set Selection Procedure}
   \label{alg:shadow}
\begin{algorithmic}
\begin{footnotesize}
   \STATE {\bfseries Input:} $\sampSet = \{x_i\}_{i=1}^{\nsamp}$, bandwidth $\s$, and
     $\ell\in\R_+$.
   \STATE Set $\shCent = \emptyset$, $\weightSet = \emptyset$, $\ncent = 0$, and
            \begin{equation} \label{shadow}
             \shrad = \s/\ell.
            \end{equation}
   \WHILE{$\sampSet \neq \emptyset$}
	 \STATE Let $c$ be first element of $\sampSet$.
     \STATE Find shadow set
	        $\shSet = \{ y \in \sampSet \ :\  \|y - c\| < \shrad\}$.
	 \STATE Update center set $\shCent = \shCent \cup \{c\}$.
	 \STATE Update weight set $\weightSet = \weightSet \cup \{\card{\shSet}\}$.
	 \STATE Set $\sampSet = \sampSet \backslash \shSet$.
   \ENDWHILE
\end{footnotesize}
\end{algorithmic}
\end{algorithm}
%\begin{wrapfigure}{r}{0.43\textwidth}
 \centering
  \includegraphics[width=0.42\textwidth,clip=true,trim=0.5in 0.05in 0.25in 0in]{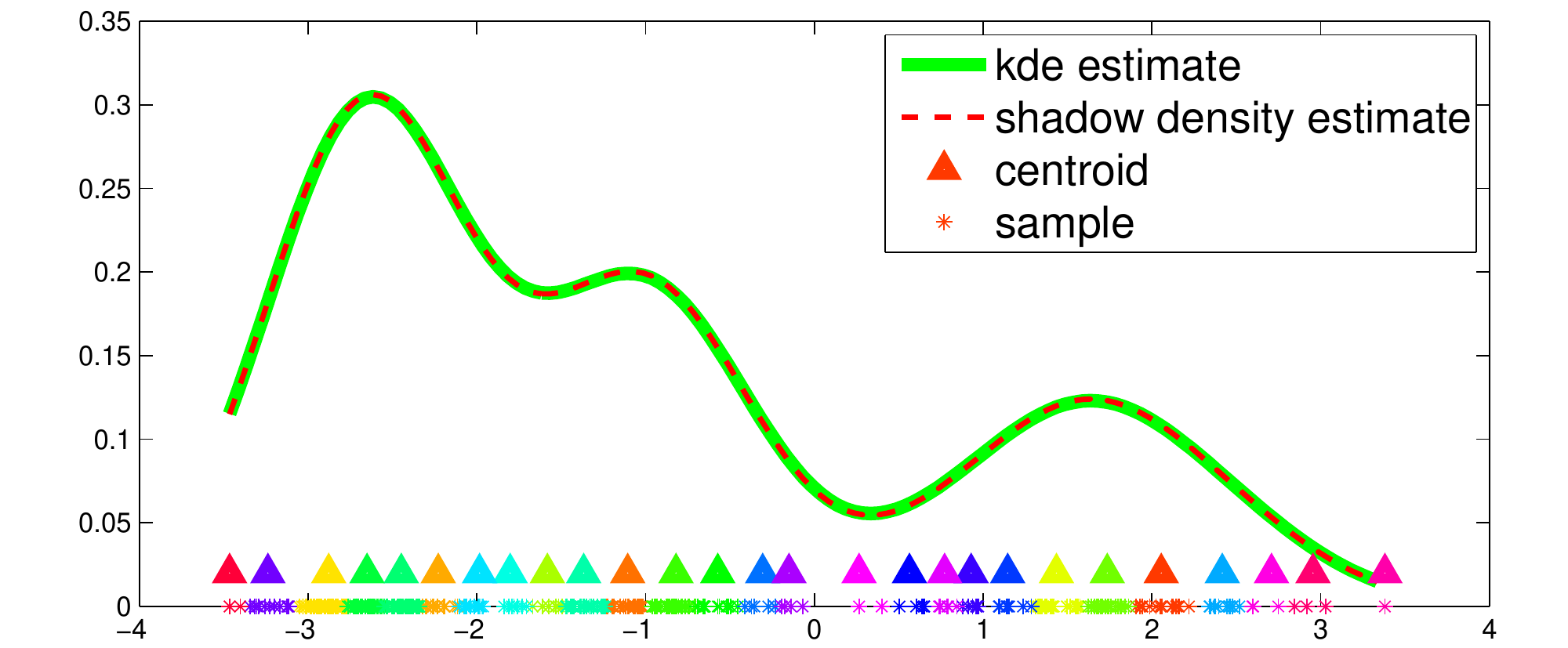}
  \caption{Visualization of the data, the shadow centers, and the associated
  KDE and ShKDE. \label{shVis}}
%\end{wrapfigure}
\end{figure}

\section{Analysis of Approximation Error}
\label{sec:bounds}
This section derives bounds on the MMD error for shadow densities,
plus bounds on the difference between the eigenvalues and spectral projections
of the operators associated to the original kernel matrix generated by KPCA,
$\empK$, and the one generated by the shadow density, $\shK$.
The bounds demonstrate the claim that an accurate RSDE leads to an accurate
eigendecomposition, since the bounds on the approximation error of the
eigendecomposition are given in terms of the error of the approximated
density estimate.

Consider a set of points $\sampSet = \sampSetLong$, sampled from the
distribution $\idealP$.  Let the shadow centers be given by $\shCent =
\shCentLong$, and define the data-to-center mapping
$\alpha:\{1,\dots,\nsamp\}\to\{1,\dots,\ncent\}$.
The shadow quantized dataset generated from $\sampSet$ is given by \
$\quaC = \{c_{\al(1)},\dots,c_{\al(\nsamp)}\}$.

Here, as in \cite{Zhang}, kernels that satisfy the following inequality are
considered,
\begin{align}\label{eq:profile}
 (k(a,b)-k(c,d))^2 \leq C_\sampSet^k(\|a-b\|^2+\|b-d\|^2),
\end{align}
where $C_\sampSet^k$ is a constant depending on $k$ and the sample
set $\sampSet$, and that can be written as
\begin{equation} \label{eq:profile2}
 k(x,y) = \vp{x-y}{\s}
\end{equation}
The Laplacian and Gaussian, in particular, satisfy \eqref{profile}
and \eqref{profile2} for $\v(s) = e^{-s}$. The constant $C_\sampSet^k$ is
$\frac{1}{\s^2}$
for the Laplacian, and is $\frac{1}{2\s^2}$ for the Gaussian \cite{Zhang08}.

The maximum mean discrepancy (MMD) is a distance measure between probability
distributions in the Hilbert space $\H$ induced by the kernel $k$ \cite{Smola}.
The (biased) MMD is defined to be
\begin{equation}\label{eq:mmd_def}
 \MMD(\sampSet,\sampSetOther)_b^2
   := \left\|\sum_{i=1}^{\nsamp}\frac{1}{\nsamp}\psi(x_i)
             - \sum_{i=1}^{\nsamp}\frac{1}{\nsamp}\psi(y_i)\right\|^2_{\H},
\end{equation}
where the $b$ denotes bias and $\fmap$ is the mapping from the input space
$\R^d$ to $\H$, $\fmap(x) := k(x, \cdot)$. The points $x_i \in \sampSet$ and
$y_i \in \sampSetOther$ are generated by probability distributions $p$ and
$q$ respectively; both sets have the same number of elements.  The MMD can be
thought of as the squared $L^2$ distance between two KDEs of the form
\eqref{k_discrete} (up to scaling factors induced by $\H$) \cite{Smola}.
Since the kernel in KPCA induces a smoothing effect on the samples from the
true probability density $\idealP$, a small value for the MMD between the
KDE and an RSDE is indicative of the RSDE acting as an effective surrogate
for $\idealP$ in the KPCA space, thus generating an effective approximation
to \eqref{opdef2} via the use of Algorithm \ref{alg:shkpca}.

The theorem below bounds the difference in MMD between
the KDE $\empP(x)$ and the ShDE $\shP(x)$.
\begin{theorem} \label{prop1a}
  \textbf{(MMD Worst Case Bound)}
  Let $\nsamp$ be the number of samples, $\sampSet$ be defined as above,
  $\ \quaC \ $ be the quantized dataset, and let $k$ satisfy \eqref{profile2}. Then
  \begin{equation} \label{mmd_result1}
    \MMD(\sampSet, \ \quaC \ )_b \leq \sqrt{2\left(\kappa - \v\left(\frac{1}{\ell^p}\right)\right)}.
  \end{equation}
\end{theorem}
\begin{proof}
 Follows from  \eqref{profile2} and \eqref{mmd_def} through the identity
 $ \sum_{c_i \in \shCent} w_i \psi(c_i) = \sum_{x_i \in \sampSet}
   \psi(c_{\alpha(i)}),$ which gives the ShDE and the KDE the same
   cardinality, $\nsamp$.
\end{proof}
The ShDE+RSKPCA procedure creates a matrix $\shK$ that acts as an $m \times m$
surrogate for the quantized kernel matrix
$\quaK_{ij} = k(\cq{i},\cq{j}),\ \text{for}\ i,j = 1 \dots \nsamp$.
Exploiting the quantization effect, the following theorems bound the
eigenvalue difference between the two spectral decompositions and also
the difference between the operators in $\H$ in terms of the Hilbert-Schmidt
norm.
\begin{theorem} \label{shadowDiffGeneric}
 Let $k$ be such that \eqref{profile} holds, and let $\la_i$ and
 $\bar{\la}_i$ be the eigenvalues of the normalized matrices $K$ and $\quaK$
 respectively. Then
 \begin{align*}
  \sum_{i=1}^{\nsamp} (\la_i - \bar{\la}_i)^2
    \leq 2C_\sampSet^k\left(\frac{\s}{\ell}\right)^2.
 \end{align*}
\end{theorem}
\begin{proof}
 Follows from the Hoffman-Wielandt inequality and \eqref{profile}.
\end{proof}
Given a kernel function $k$ and $\X$, a finite dimensional operator
$K_{\nsamp}:\H\to\H$ approximating the ideal operator \eqref{opdef2}
can be defined via
\begin{align} \eqlabel{extrap_op}
 K_{\nsamp}(\cdot) := \frac{1}{\nsamp}\sum_{i=1}^{\nsamp}\l\cdot,k_{x_i}\r_{\H} k_{x_i},
\end{align}
where $x_i\in\X$ and $\l\cdot,k_{x_i}\r_{\H}$ projects the point onto the kernel function
$k_{x_i} := k(\cdot,x_i) \in\H$ \cite{Rosasco:2010}. The operator can be used
to bound the error in Hilbert-Schmidt norm between the empirical operators
generated by KPCA and ShDE+RSKPCA.
\begin{theorem} \label{prop_hschmidt}
 Let $K_{\nsamp}$ and $\ \quaK_{\nsamp}$ be defined using \eqref{extrap_op}
 with $\sampSet$ and $\ \quaC$, respectively.
 Then
 \begin{align} \label{eq:hs_result1}
  \|K_{\nsamp} - \quaK_{\nsamp}\|_{\HS} \leq 2\kappa\sqrt{2\left(\kappa - \v\left(\frac{1}{\ell^p}\right)\right)}.
 \end{align}
\end{theorem}
\begin{proof}
 Define the operators $K_{\nsamp}$ and $ \ \quaK_{\nsamp} $ via the extrapolation \eqref{extrap_op} using
 $k_{x_i}$ and $k_{\cq{i}}$ respectively, and define the kernel residual in $\H$ to be $\e_i := k_{x_i} - k_{\cq{i}}$.
 Then
 \begin{align*}
  K_{\nsamp} - \quaK_{\nsamp} &= \frac{1}{\nsamp}\sum_{i=1}^{\nsamp}\left( \l\cdot,k_{x_i}\r_{\H} \e_i + \l\cdot,\e_i\r_{\H} k_{\cq{i}} \right),
 \end{align*}
leading to
 \begin{align*}
\left\|K_{\nsamp} - \quaK_{\nsamp}\right\|_{\HS}
            &\leq
			\left\|\frac{1}{\nsamp}\sum_{i=1}^{\nsamp}\l\cdot,k_{x_i}\r_{\H}
			\e_i\right\|_{\HS} 
                + \left\|\frac{1}{\nsamp}\sum_{i=1}^{\nsamp}\l\cdot,\e_i\r_{\H} k_{\cq{i}} \right\|_{\HS}.
 \end{align*}
 Using the properties of the Hilbert-Schmidt norm, and the maximizer $\e'$
 such that the centroid error $\|\e_i\|_{\H}$ is largest, the theorem follows.
\end{proof}
Proposition \ref{prop_hschmidt} shows that the centroid error in $\H$ is the
key to the performance of the learning algorithm, and that the error
is controlled solely in terms of the parameter $\ell$. The independence of
the performance from the weights shows that ShDE effectively learns the
percentage of the data that needs to be retained based on the value of $\ell$,
which is dependent on the kernel and not the data.  Finally, $\ell$ controls
both the MMD and operator approximations, implying that the density
estimate used in the shadow density procedure is sensible for learning in
the eigenspace. Using this result, the following theorem follows.
\begin{theorem} \label{prop_eigenspace}
 Let $K_{\nsamp}$ and $\quaK_{\nsamp}$ be symmetric positive (finite) Hilbert-Schmidt
 operator of $\H$ defined by \eqref{extrap_op}, and
 assume that $K_{\nsamp}$ has simple nonzero eigenvalues $\la_1 > \la_2 > \dots >
 \la_n$. Let $D > 0$ be an integer such that $\la_D > 0, \ \delta_D = \frac{1}{2}\left(\la_D-\la_{D+1}\right)$.
 If $2\sqrt{\kappa}\|\e'\|_{\H} < \delta_D/2$, then
\begin{align}
 \|P^D(K_{\nsamp})-P^D(\quaK_{\nsamp})\|_{\HS} \leq \frac{2\sqrt{2\kappa\left(\kappa - \v\left(\frac{1}{\ell^p}\right)\right)}}{\delta_D},
\end{align}
where $P^D(A)$ denotes the projection onto the $D$-dimensional eigenspace of $A\in\HS(\H)$ associated to the largest eigenvalues.
\end{theorem}
\begin{proof}
 Follows from Theorem 3 in \cite{Zwald:2005} and Proposition
 \ref{prop_hschmidt}.
\end{proof}

\section{Experimental Results}
\label{sec:experiments}

This section demonstrates the effectiveness of RSKPCA on real-world data.
Approximation accuracy tests include eigenembedding and classification tasks
with the Gaussian kernel.  The datasets used and the bandwidths chosen
(via cross-validation) are given in Table \ref{table1}.
In the figures, $\ntrain$ refers to the
number of the points the model is trained on.  All of the comparison
algorithms require specification of the reduced set size $\ncent$.  To
compare, the shadow method is run with $\ell$ then the average of all
$\ncent$ achieved on the datasets determines the value $\ncent$ for the
other methods.  Table \ref{table2} compares the training time and storage
size (which relates to evaluation time).  All comparisons are made with KPCA
as the baseline. Speedup is relative to the equivalent KPCA execution time.

\vspace*{-0.075in}
\paragraph{Eigenembedding comparison with Nystr\"{o}m methods.}
This experiment demonstrates the fidelity of the eigenfunctions computed by
ShDE+RSKPCA to those generated by KPCA.  The capacity of
generalization of the approximate eigenfunctions is tested. Using KPCA as
the baseline, ShDE+RSKPCA is compared with three other methods:
1) subsampled KPCA with bases chosen via random uniform sampling,
2) the regular Nystr\"{o}m method with bases chosen via random uniform
sampling, and
3) the density weighted Nystr\"{o}m (WNystr\"{o}m) method \cite{Zhang}.
The experimental methodology is as follows. First, the KPCA model is trained
on the entire dataset. Then, shadow, uniform,  Nystr\"{o}m, and WNystr\"{o}m
KPCA models are trained using $80\%$ of the data for $\ell\in[3.0,5.0]$, in
increments of $0.1$. The reason $\ell = 3.0$ is chosen as a lower bound for the
Gaussian is because lower values of $\ell$ pick points that are no longer
similar to the centroid, while $\ell > 5$ generally results in a loss in training efficiency.
  The KPCA eigenfunction embedding is computed for the
remaining $20\%$ of the data for all the models, with rank $\neigs=5$.  The
embeddings are aligned with each other using the transform
${\rm argmin}_{A\in\R^{\neigs\times\neigs}}\|O-\tilde{O}A\|_F$, where $O$ is
the matrix representing the KPCA embedding, and $\tilde{O}$ represents the
approximate KPCA embedding. The Frobenius norm difference of the embeddings
and eigenvalues, the training and testing speedup, and the amount of data
retained are averaged over 50 runs for each $\ell$, and are shown in Figures
\ref{fig_ny_german} and \ref{fig_ny_pendigits} for the \textbf{german} and
\textbf{pendigits} datasets.  As expected, while subsampled KPCA is faster
in the training stage, it performs worse than any other method, implying
that an appropriate weighting is necessary to approximate the eigenfunctions
of KPCA. For larger values of $\ell$, ShDE+RSKPCA always performs well when
it comes to approximating the eigenvalues and eigenfunctions of the
operator.  In terms of eigenembedding accuracy, using ANOVA with a value of
$\al = 0.05$, ShDE+RSKPCA is better than the Nystr\"{o}m embeddings after
$\ell = 3.2 (3.3)$ and no worse than the WNystr\"{o}m embeddings after
$\ell = 4.0 (4.8)$ for \textbf{pendigits} (\textbf{german}), and asymptotically
approaches the KPCA baseline.
While slower than the Nystr\"{o}m method for training, ShDE+RSKPCA is faster
than KPCA for training and achieves significant testing speedups.
It does so by retaining a subset of the data via selection of $\ell$, c.f.
Fig. \ref{retained}(a,b).

\begin{figure}[t]
\centering
\subfigure[Eigenvalue Deviation]{
  \includegraphics[scale=0.185,clip=true,trim=0.1in 0.05in 0.5in  0.4in]{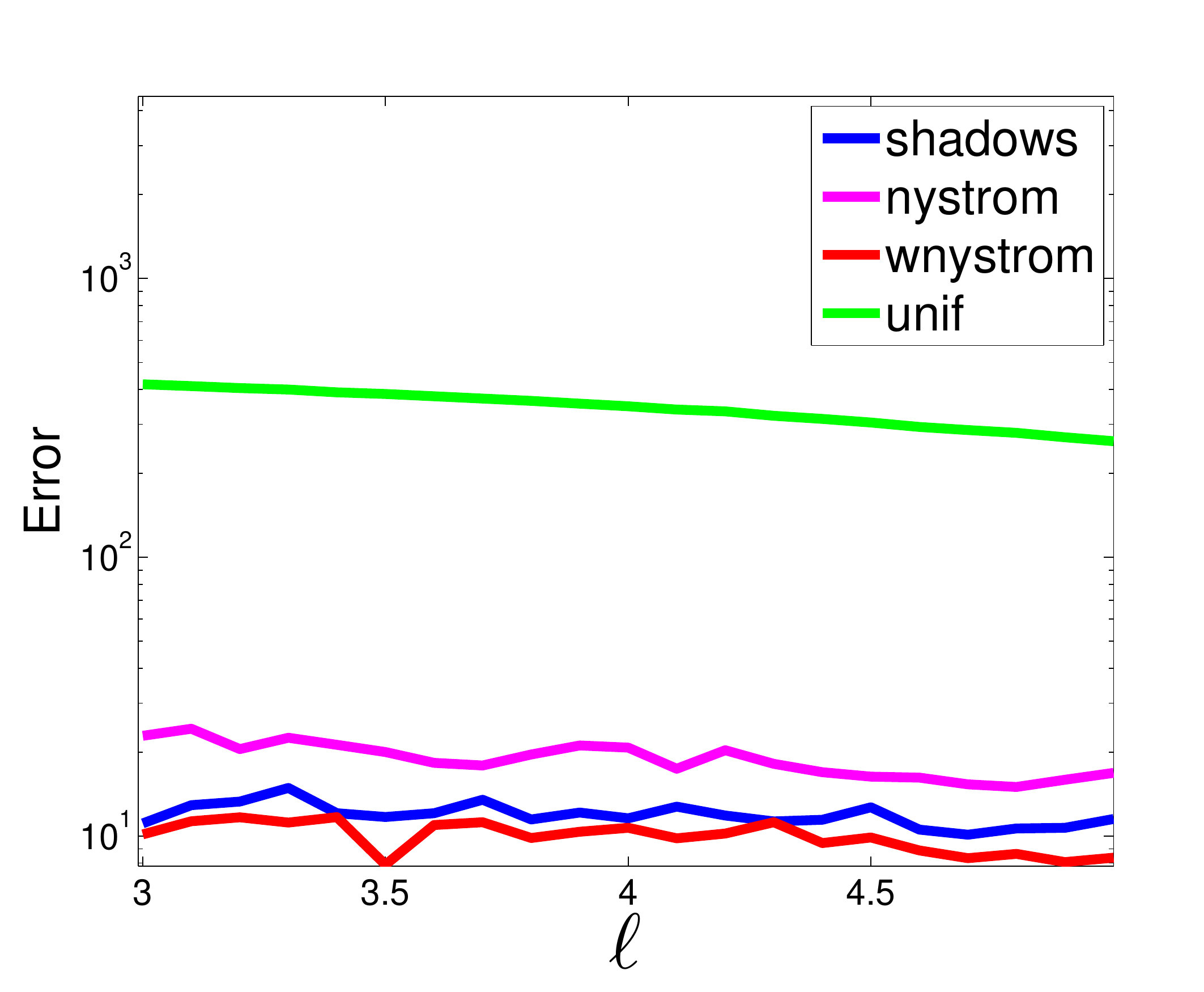}
  %\label{fig1:subfig1}
}
\subfigure[Embedding Accuracy]{
  \includegraphics[scale=0.185,clip=true,trim=0.1in 0.05in 0.5in  0.4in]{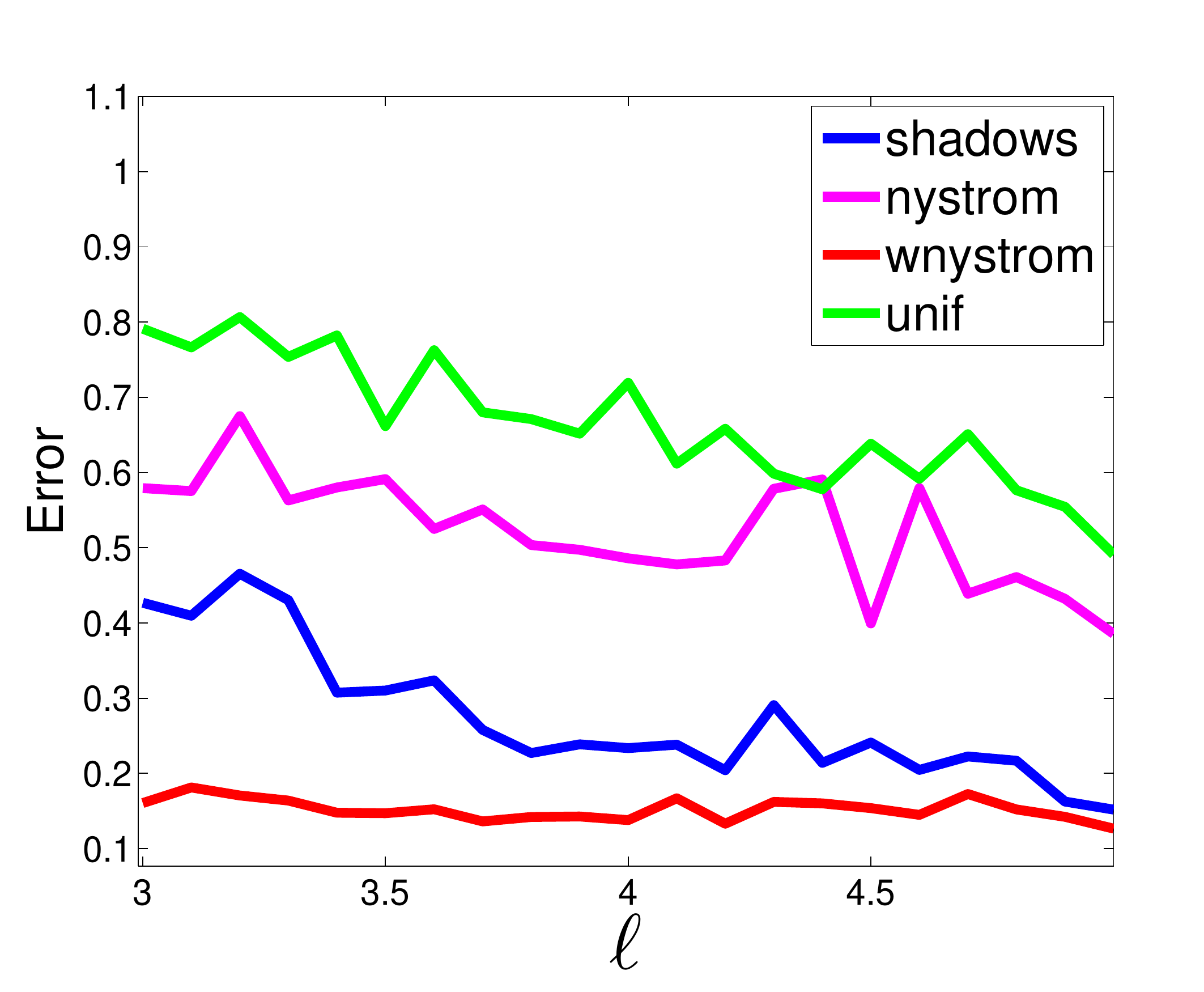}
  %\label{fig1:subfig1}
}
\subfigure[Eigen speedup]{
  \includegraphics[scale=0.185,clip=true,trim=0.1in 0.05in 0.5in  0.4in]{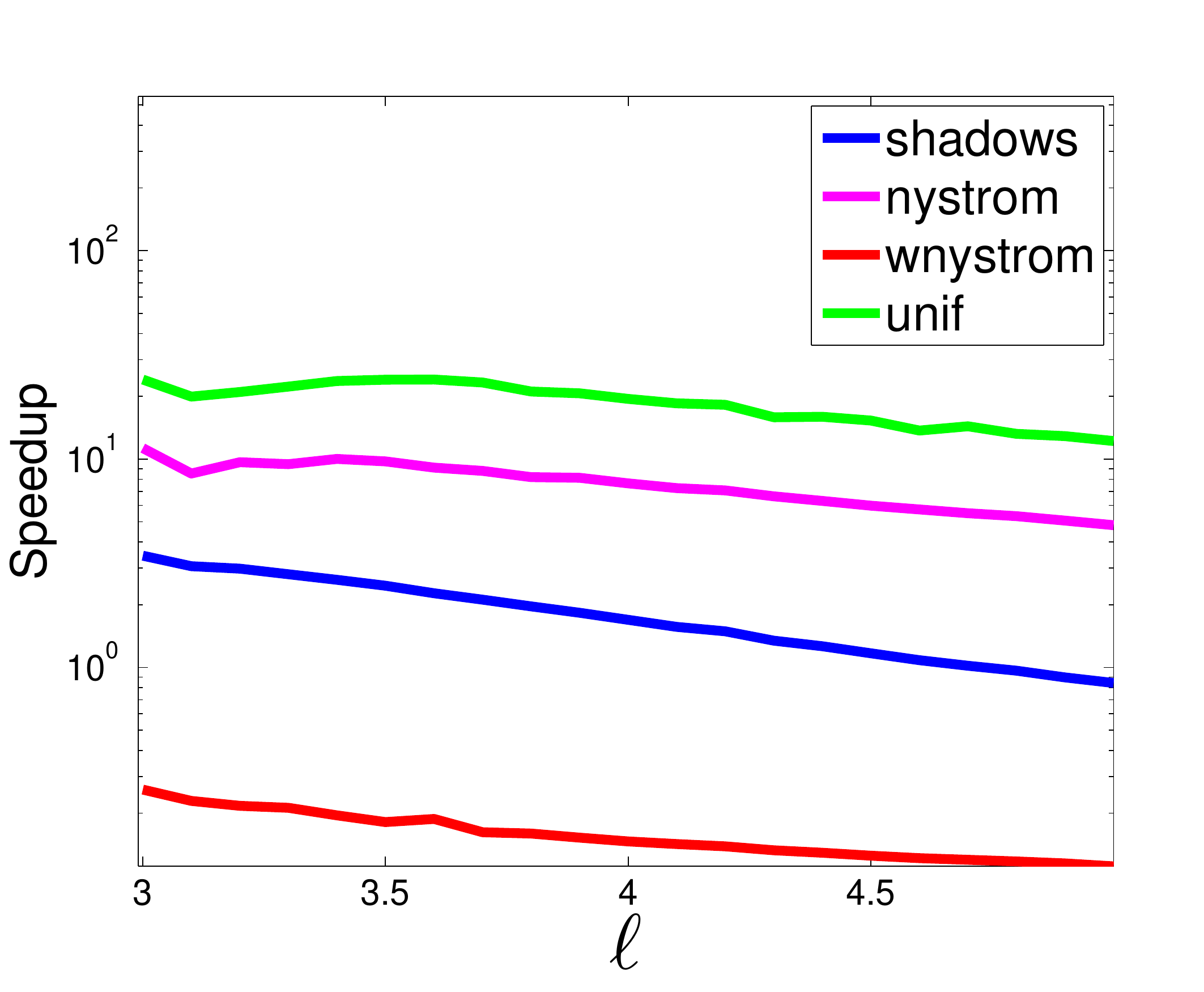}
}
\subfigure[Testing speedup]{
  \includegraphics[scale=0.185,clip=true,trim=0.1in 0.05in 0.5in  0.4in]{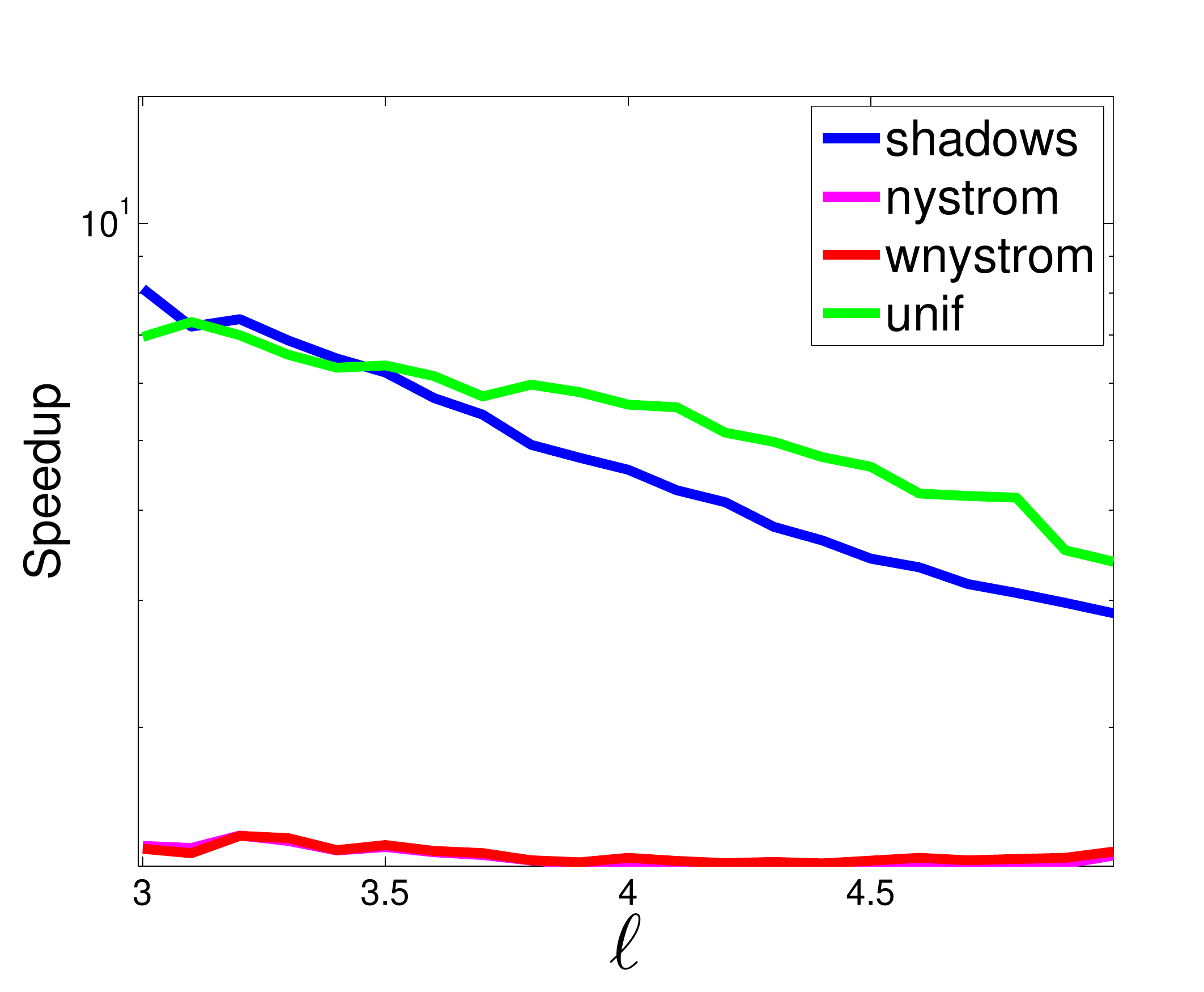}
  %\label{fig1:subfig1}
}
\caption{Eigenembedding comparison w/Nystr\"{o}m methods for \textbf{german} as $\ell$ is
  varied ($\ntrain = 800$).}
\label{fig_ny_german}
\end{figure}

\begin{figure}[ht!]
\centering
\subfigure[Eigenvalue Deviation]{
  \includegraphics[scale=0.185,clip=true,trim=0.1in 0.05in 0.5in  0.4in]{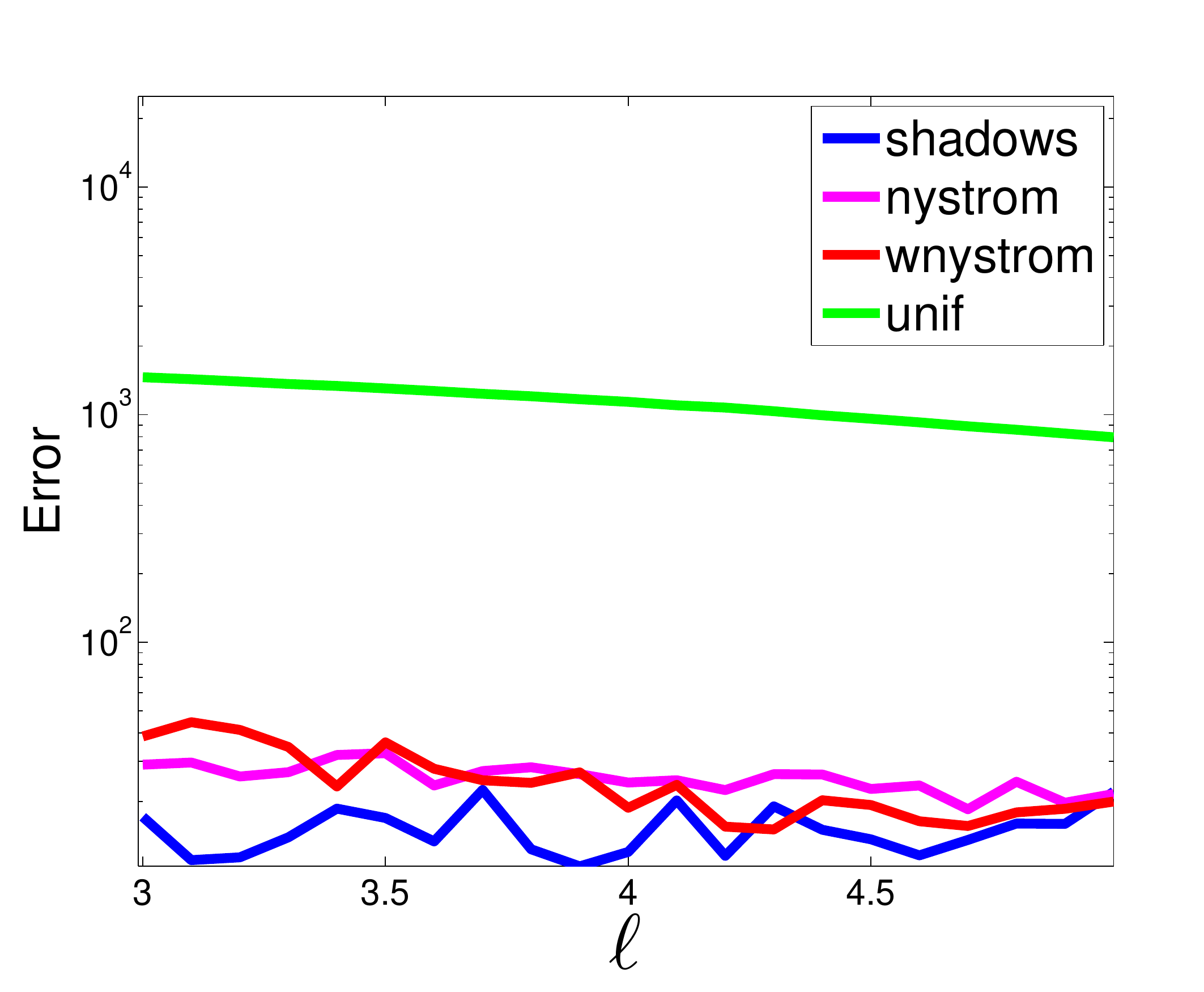}
  %\label{fig1:subfig1}
}
\subfigure[Embedding Accuracy]{
  \includegraphics[scale=0.185,clip=true,trim=0.1in 0.05in 0.5in  0.4in]{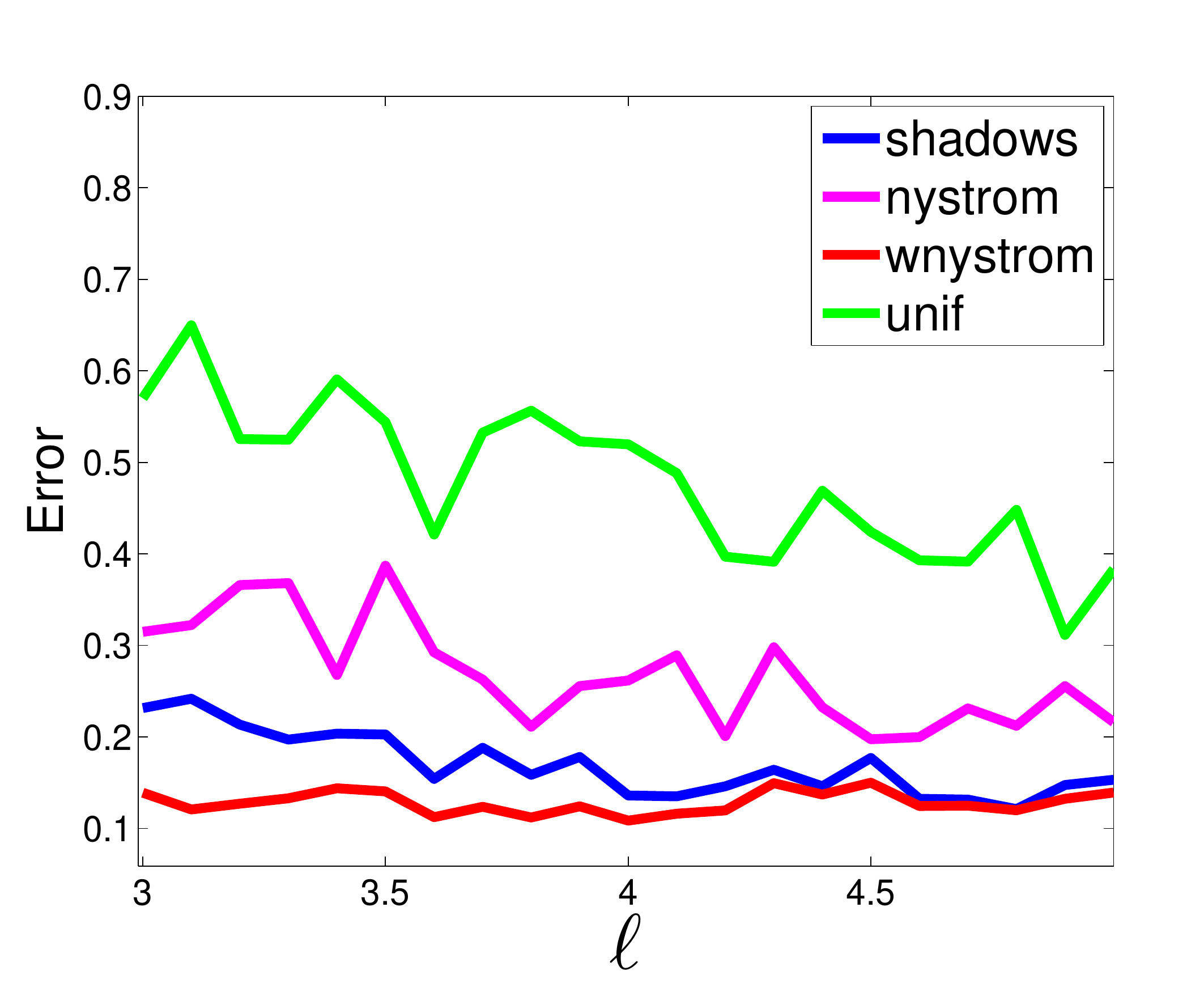}
  %\label{fig1:subfig1}
}
\subfigure[Eigen speedup]{
  \includegraphics[scale=0.185,clip=true,trim=0.1in 0.05in 0.5in  0.4in]{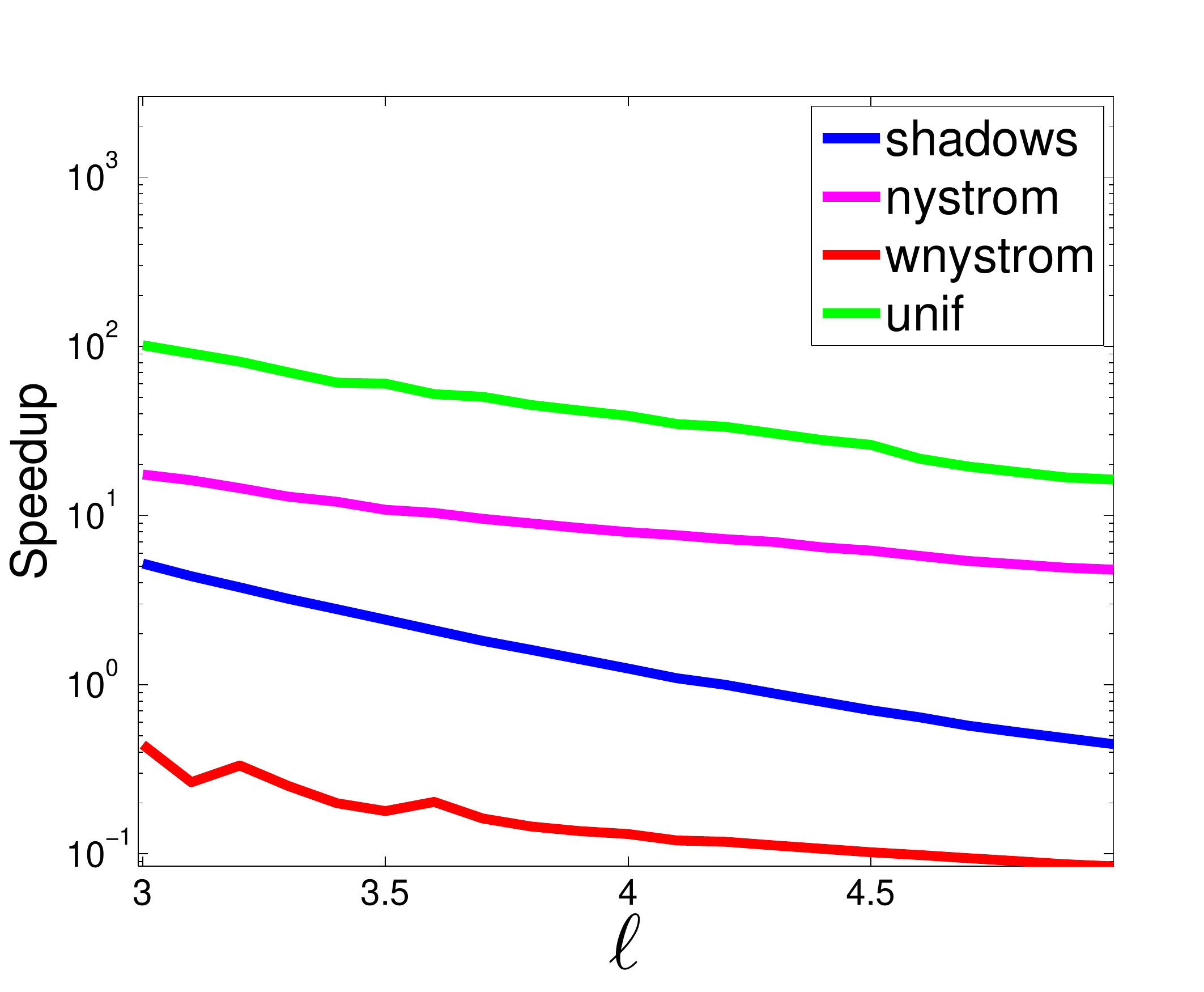}
}
\subfigure[Testing speedup]{
  \includegraphics[scale=0.185,clip=true,trim=0.1in 0.05in 0.5in  0.4in]{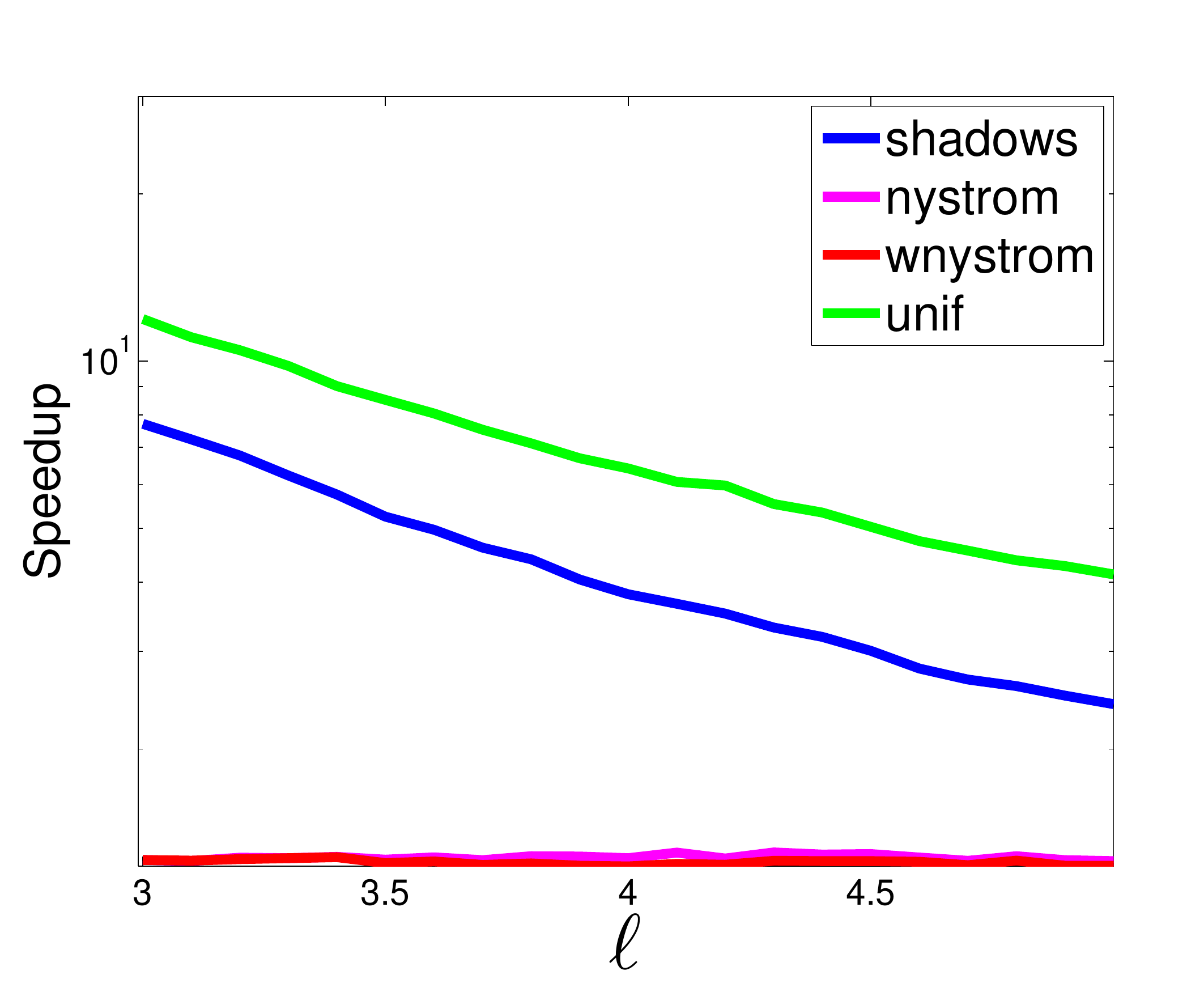}
  %\label{fig1:subfig1}
}
\caption{Eigenembedding comparison w/Nystr\"{o}m for \textbf{pendigits} as $\ell$ is
  varied  ($\ntrain = 2,800$).}
\label{fig_ny_pendigits}
\end{figure}

\begin{table}[t]
\setlength{\tabcolsep}{4pt}
\begin{minipage}{0.5\textwidth}
\caption{Datasets used.  \label{table1}}
\vspace*{-0.1in}
\begin{center}
\begin{scriptsize}
\begin{sc}
\fbox{
\begin{tabular}{l|c|c|c|c}
                  &  \textbf{german}         & \textbf{pendigits}    & \textbf{usps} & \textbf{yale} \\
$\nsamp$          & 1,000   & 3,500 & 9,298 &  5,768 \\
dim               & 24      & 16   & 256   & 520 \\
classes           & 2      & 10     & 10    & 10 \\
$k$               & 5       & 5     & 15    & 10 \\
$\s$              & 30     & 120   & 18    & 17 \\
%\hline
\end{tabular}}
\end{sc}
\end{scriptsize}
\end{center}
\vskip -0.2in
%\end{table}
\end{minipage}
\begin{minipage}{0.5\textwidth}
%\begin{table}[h]
\vspace*{-0.15in}
\setlength{\tabcolsep}{4pt}
\caption{Training cost and storage comparison. \label{table2}}
\vspace*{-0.1in}
\begin{center}
\begin{scriptsize}
\begin{sc}
\fbox{
\begin{tabular}{l|c|c|c}
                  &  ShDE+RSKPCA   & Nystr\"{o}m    & WNystr\"{o}m \\
time &  $\O(\ncent\nsamp + \ncent^3)$ & $\O(\ncent\nsamp + \ncent^3)$ & $\O(\ncent\nsamp + \ncent^3)$\\
space       &  $\O(\ncent \neigs)$      & $\O(\nsamp\neigs)$       & $\O(\nsamp\neigs)$      \\
%\hline
\end{tabular}}
\end{sc}
\end{scriptsize}
\end{center}
\vskip -0.2in
\end{minipage}
\vspace*{-0.2in}
\end{table}

\paragraph{KPCA classification comparison with Nystr\"{o}m methods.}
This experiment examines the effectiveness of ShDE+RSKPCA for classification
compared with the Nystr\"{o}m methods used previously.
Classification utilizes the $k$-nn classifier with $k=3$, using $10$-fold
cross-validation. The accuracy, training and testing speedups, and the
percentage of data are reported. The results are shown in Figures
\ref{fig_ny_usps} and \ref{fig_ny_yale} for the \textbf{usps} and
\textbf{yale} methods respectively (none = KPCA).  For the $k$-nn
classification case, ShDE+RSKPCA has competitive accuracy with the Nystr\"{o}m
methods, while providing significant training and testing speedups. The
training speedup over the Nystr\"{o}m method in this case is because the
eigenembedding of the data needs to be computed as part of the $k$-nn
classifier training.  Note that the data retained here, Fig.
\ref{retained}(c,d), is less than 10\% for $\ell \in [3,5]$, implying
noticeable speedup in the KPCA step of the classifier (during training and
evaluation).

\vspace*{-0.025in}
\paragraph{RSKPCA with different RSDE schemes.}
\vspace*{-0.025in}
RSKPCA is performed using alternative RSDEs to demonstrate the
influence of the RSDE algorithm on accuracy, Figs. \ref{fig_center_usps}
and \ref{fig_center_yale}.
Following \cite{Zhang}, $k$-means provides a means to generate an RSDE at a
time complexity of $\O(\ncent\nsamp)$ (but tends to be slow due
being iterative).  Second, KDE paring \cite{Freedman} subsamples from
the original dataset and computes the estimate from the reduced set, at
an $\mathcal{O}(\ncent)$ cost. Third, kernel herding is examined
\cite{Chen:2010}, which provides a mechanism to sample from a KDE using a
nonlinear dynamical system.  The samples are shown to be good representative
samples.  Their generation is $\O(\nsamp^2 \ncent)$.  All of these
algorithms require the user to provide the number $\ncent$.  It can be seen
that the quality of the RSDE does influence the accuracy for small $\ell$,
less so for larger $\ell$.  The center selection schemes that lead to
improved accuracy are costlier than ShDE, thus decreasing training gains.
Evaluation speedup is the same for all methods.

\begin{figure}[ht!]
\centering
\subfigure[Accuracy]{
  \includegraphics[scale=0.212,clip=true,trim=0.1in 0.05in 0.75in  0.4in]{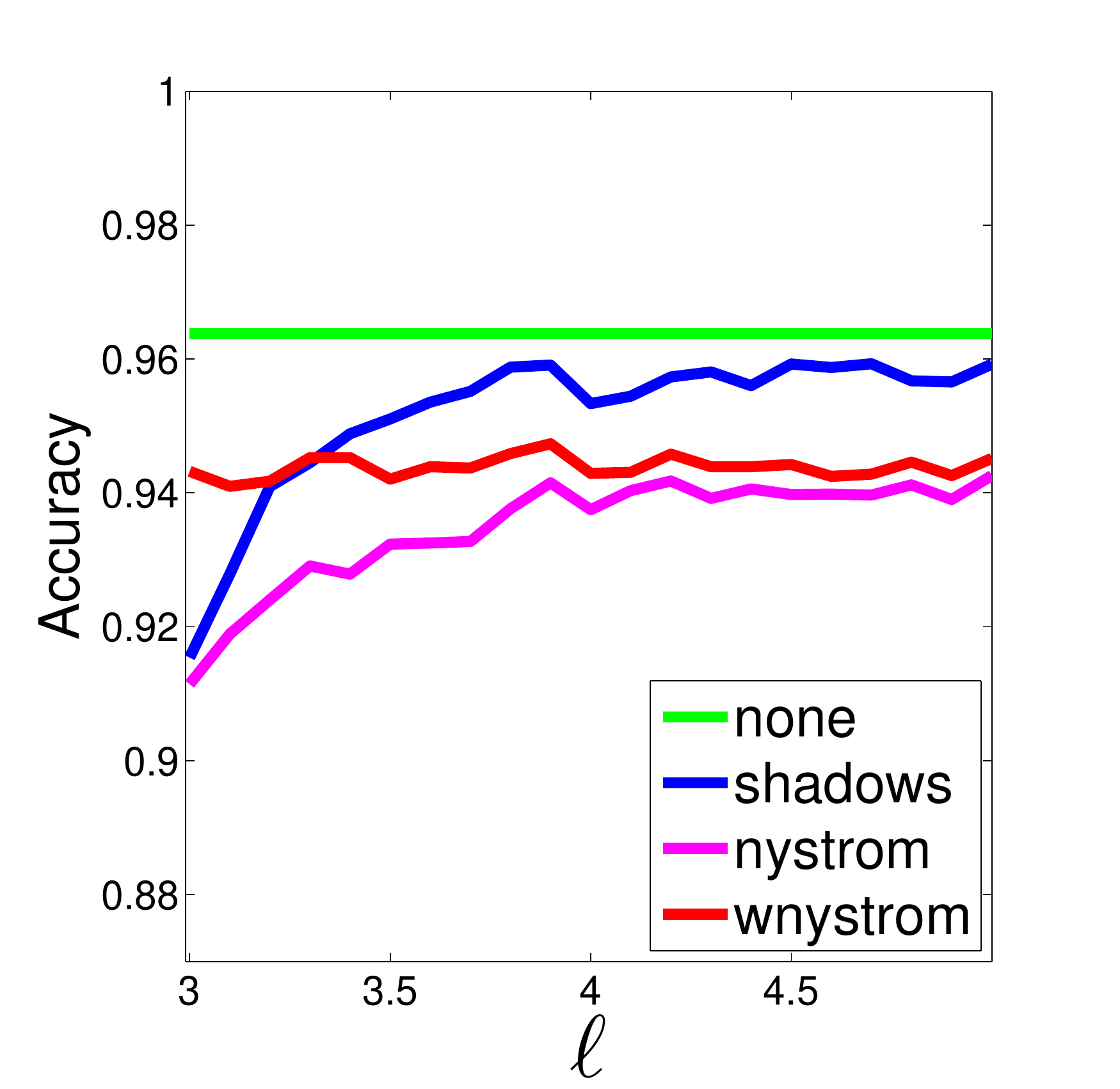}
  %\label{fig1:subfig1}
}
\subfigure[Training speedup]{
  \includegraphics[scale=0.212,clip=true,trim=0.1in 0.05in 0.75in  0.4in]{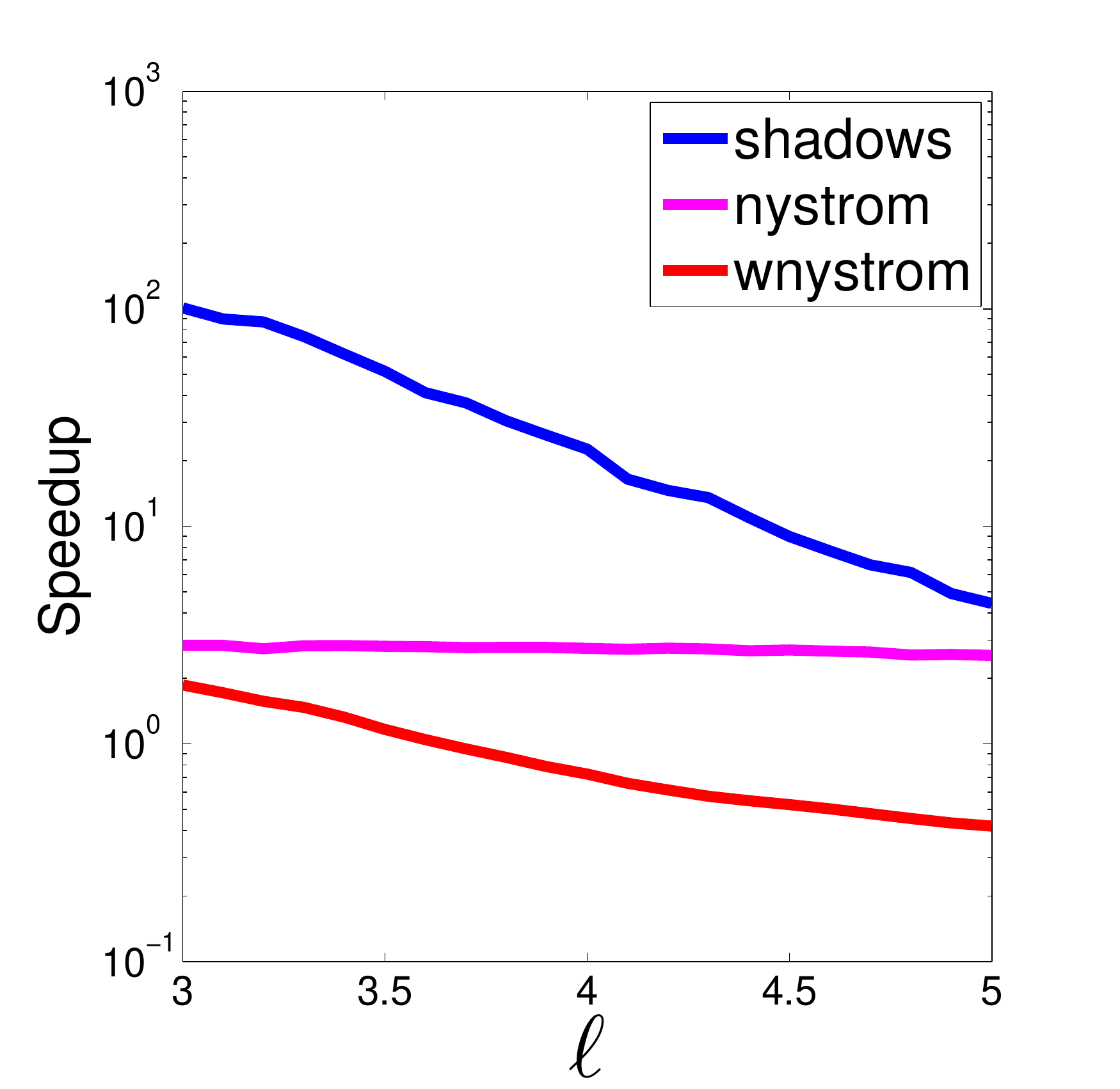}
}
\subfigure[Testing speedup]{
  \includegraphics[scale=0.212,clip=true,trim=0.1in 0.05in 0.75in  0.4in]{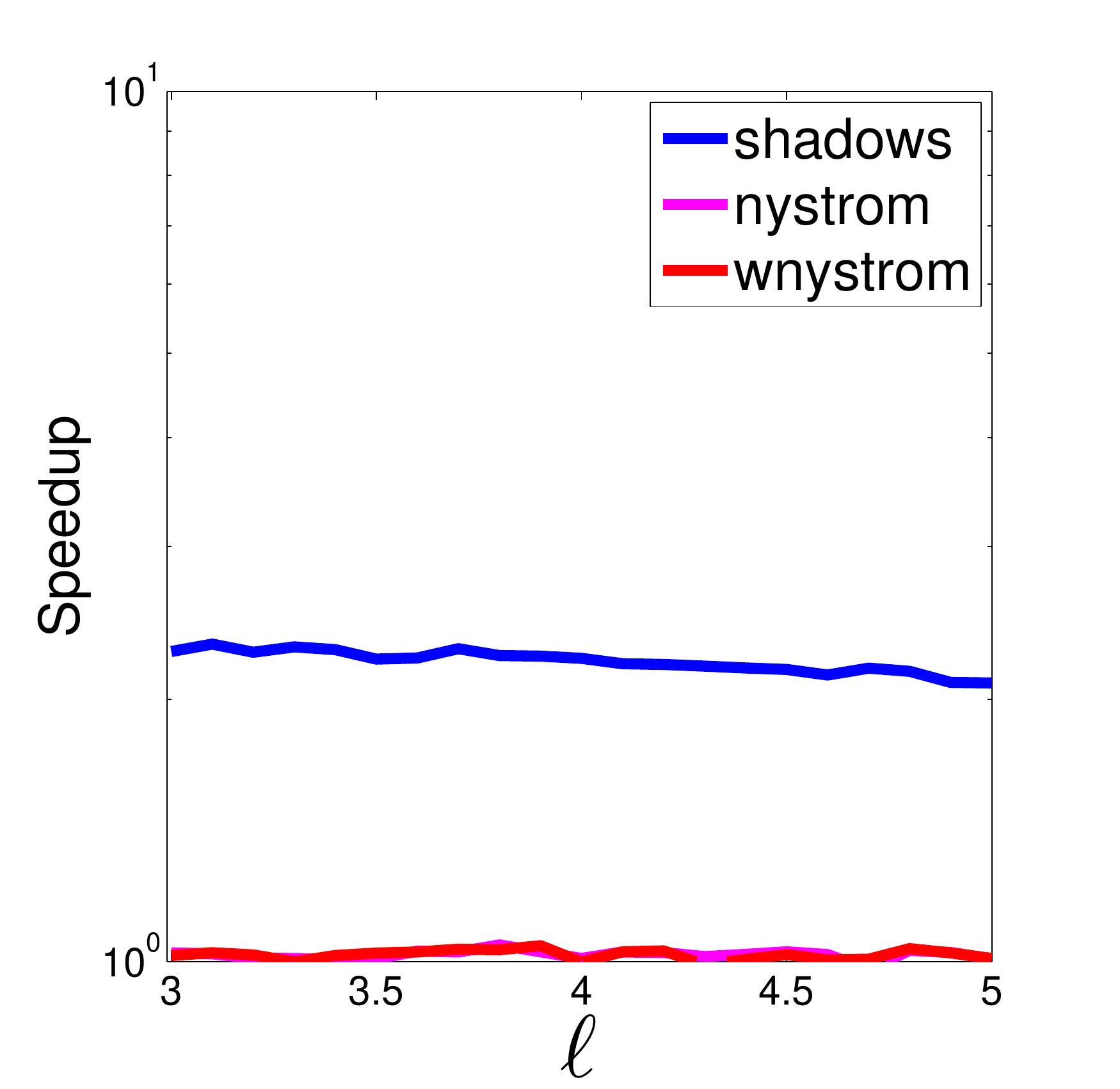}
  %\label{fig1:subfig1}
}
\subfigure[Total speedup]{
  \includegraphics[scale=0.212,clip=true,trim=0.1in 0.05in 0.75in  0.4in]{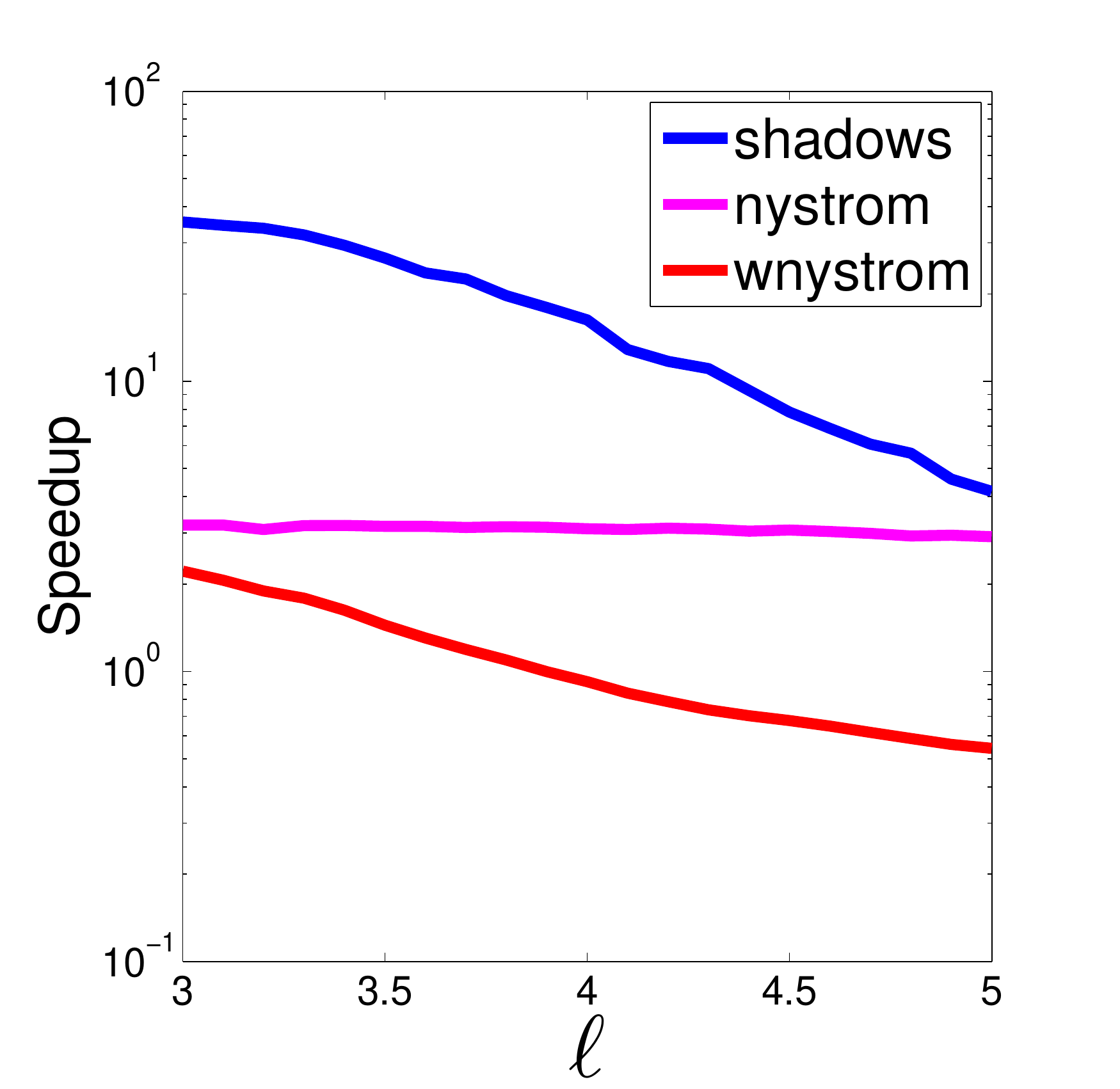}
  %\label{fig1:subfig1}
}
\vspace*{-0.10in}
\caption{Classification comparison w/Nystr\"{o}m for \textbf{usps} as $\ell$ is
  varied ($\ntrain = 8,368$).}
\label{fig_ny_usps}
\end{figure}

\begin{figure}[ht!]
\centering
\subfigure[Accuracy]{
  \includegraphics[scale=0.212,clip=true,trim=0.1in 0.05in 0.75in  0.4in]{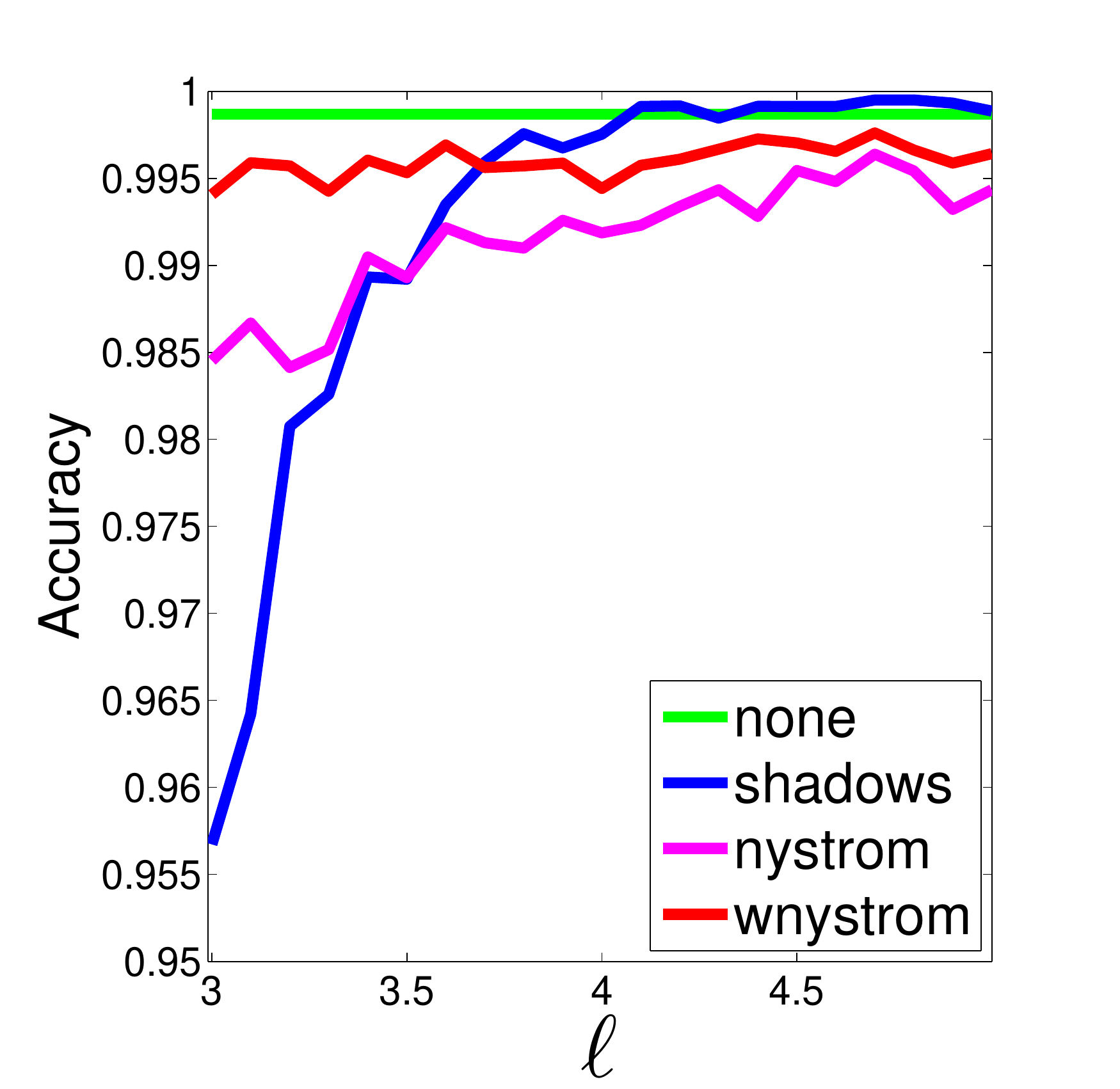}
  %\label{fig1:subfig1}
}
\subfigure[Training speedup]{
  \includegraphics[scale=0.212,clip=true,trim=0.1in 0.05in 0.75in  0.4in]{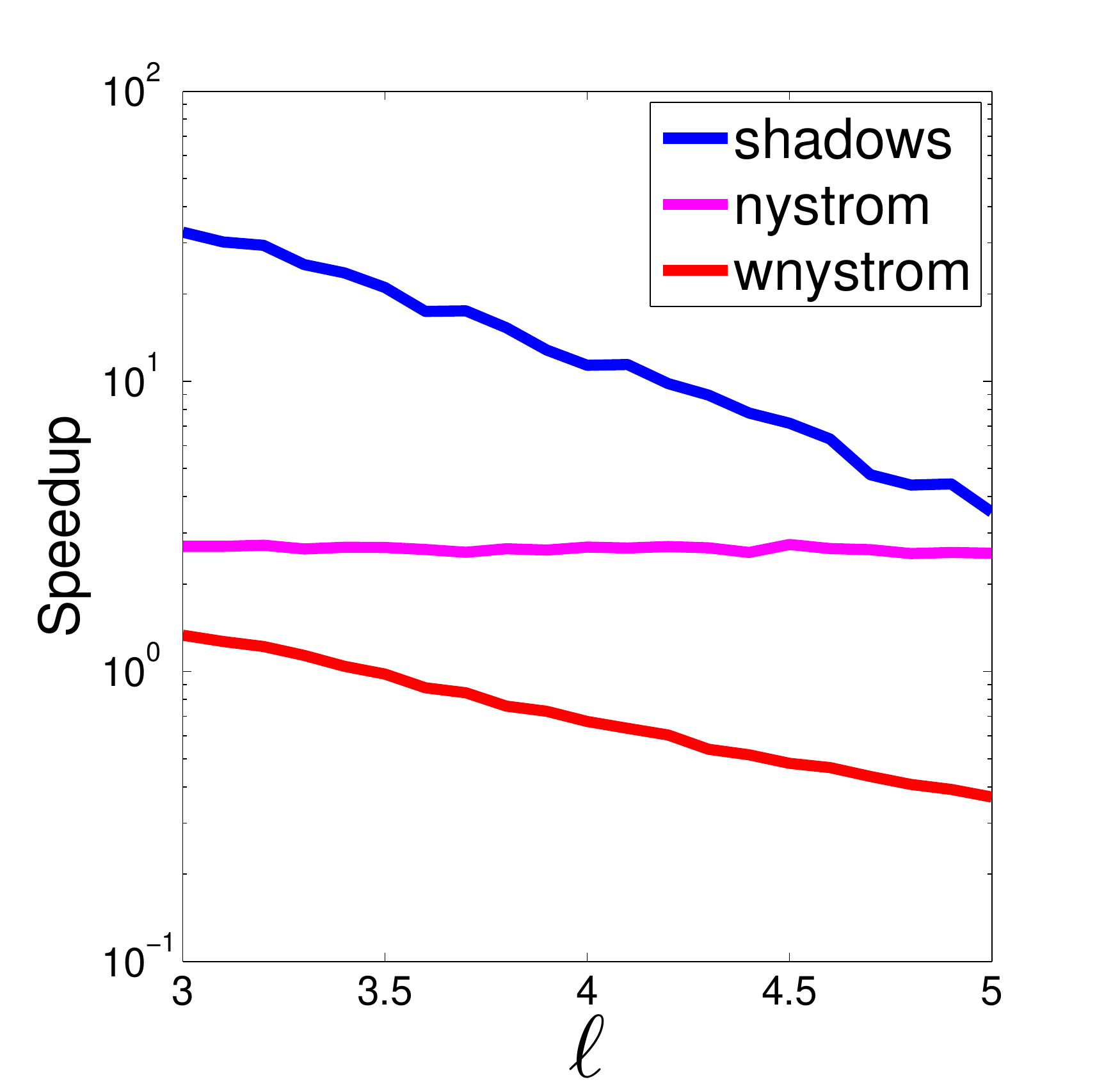}
}
\subfigure[Testing speedup]{
  \includegraphics[scale=0.212,clip=true,trim=0.1in 0.05in 0.75in  0.4in]{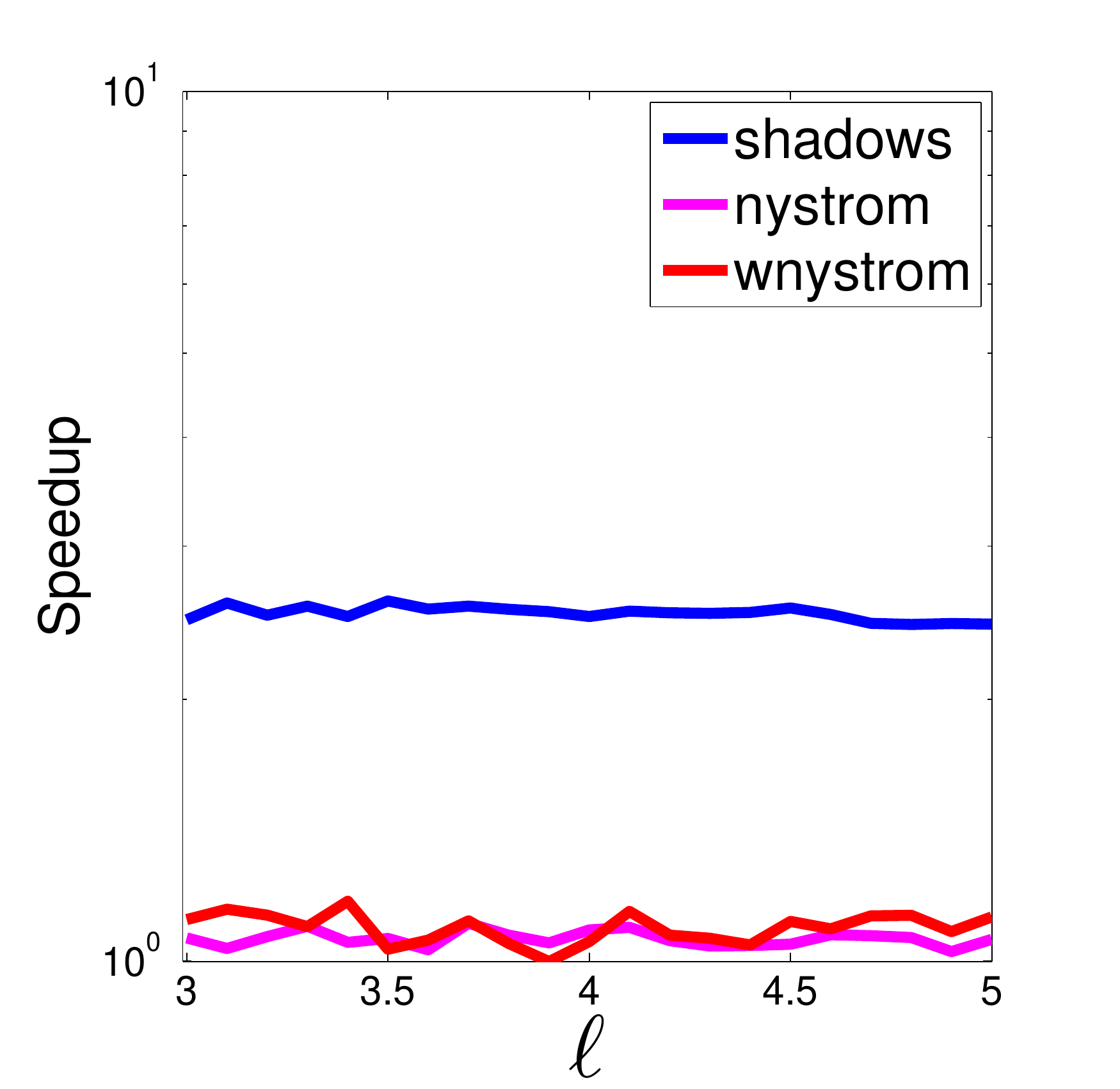}
  %\label{fig1:subfig1}
}
\subfigure[Total speedup]{
  \includegraphics[scale=0.212,clip=true,trim=0.1in 0.05in 0.75in  0.4in]{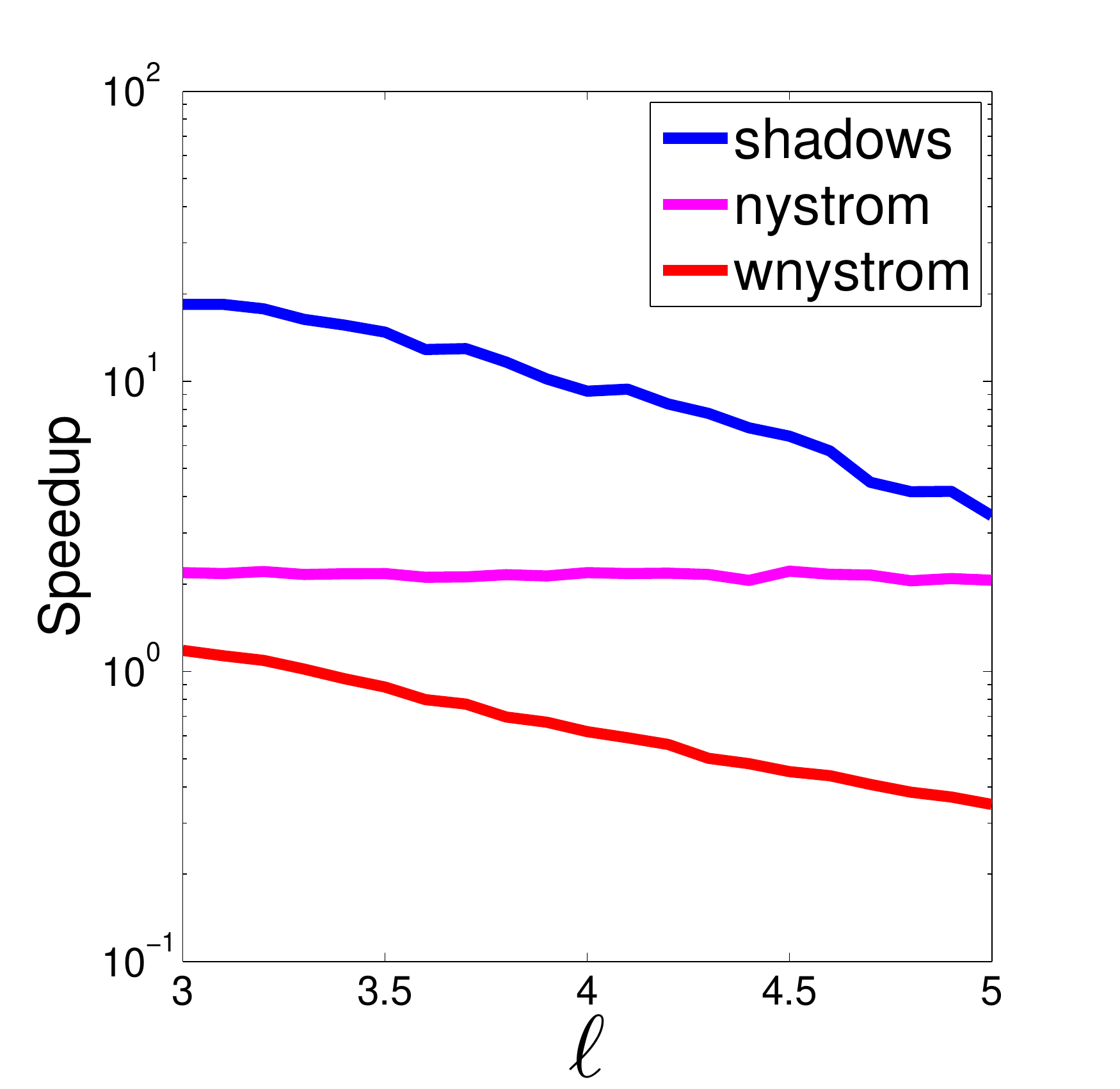}
  %\label{fig1:subfig1}
}
\vspace*{-0.10in}
\caption{Classification comparison w/Nystr\"{o}m methods for \textbf{yale} as $\ell$ is
  varied ($\ntrain = 5,191$).}
\label{fig_ny_yale}
\end{figure}

\begin{figure}[h!]
\centering
\subfigure[\textbf{german}]{
  \includegraphics[width=0.22\columnwidth,clip=true,trim=0.1in 0.05in 0.4in  0.4in]{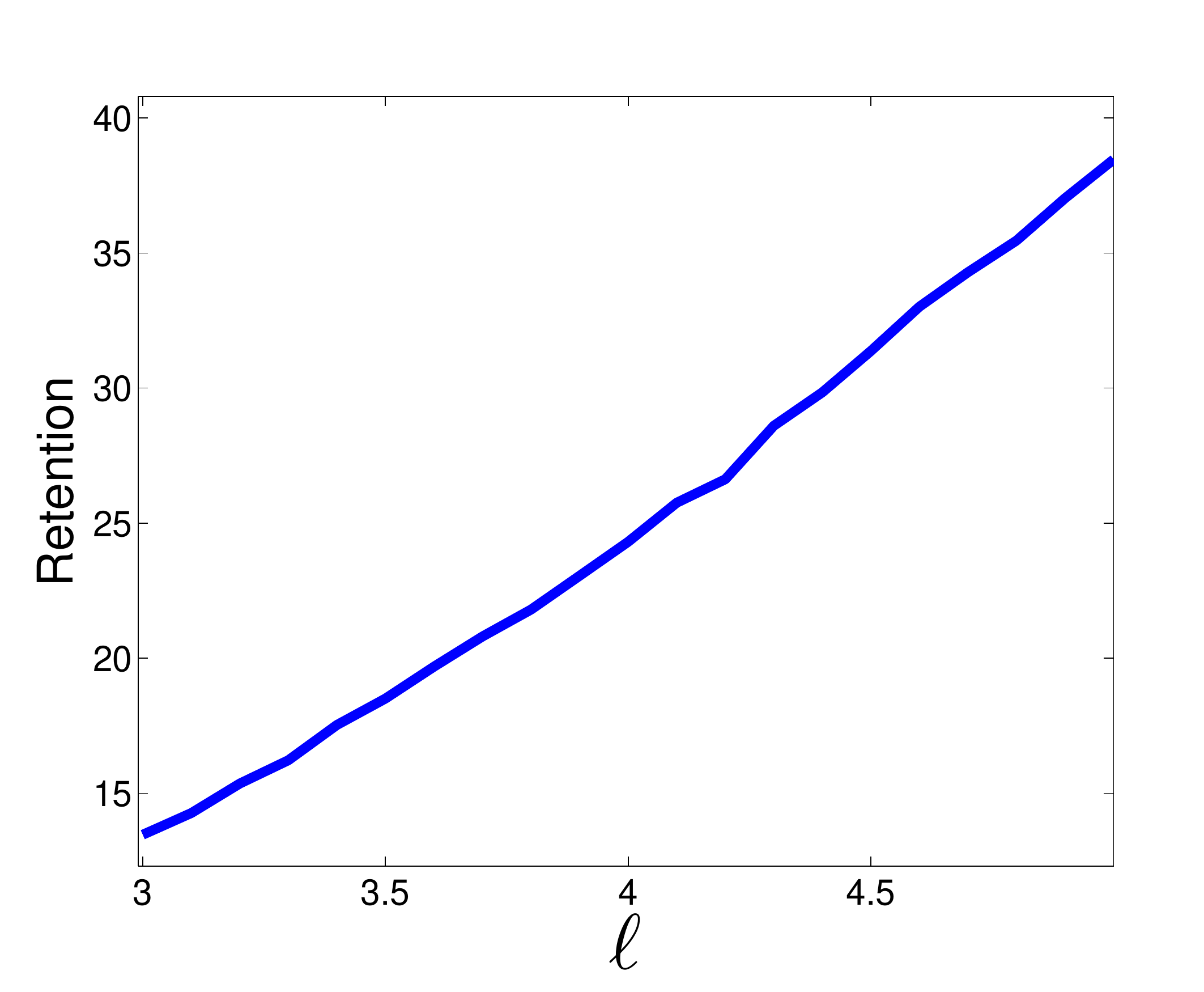}
  %\label{fig1:subfig1}
}
\subfigure[\textbf{pendigits}]{
  \includegraphics[width=0.22\columnwidth,clip=true,trim=0.1in 0.05in 0.4in  0.4in]{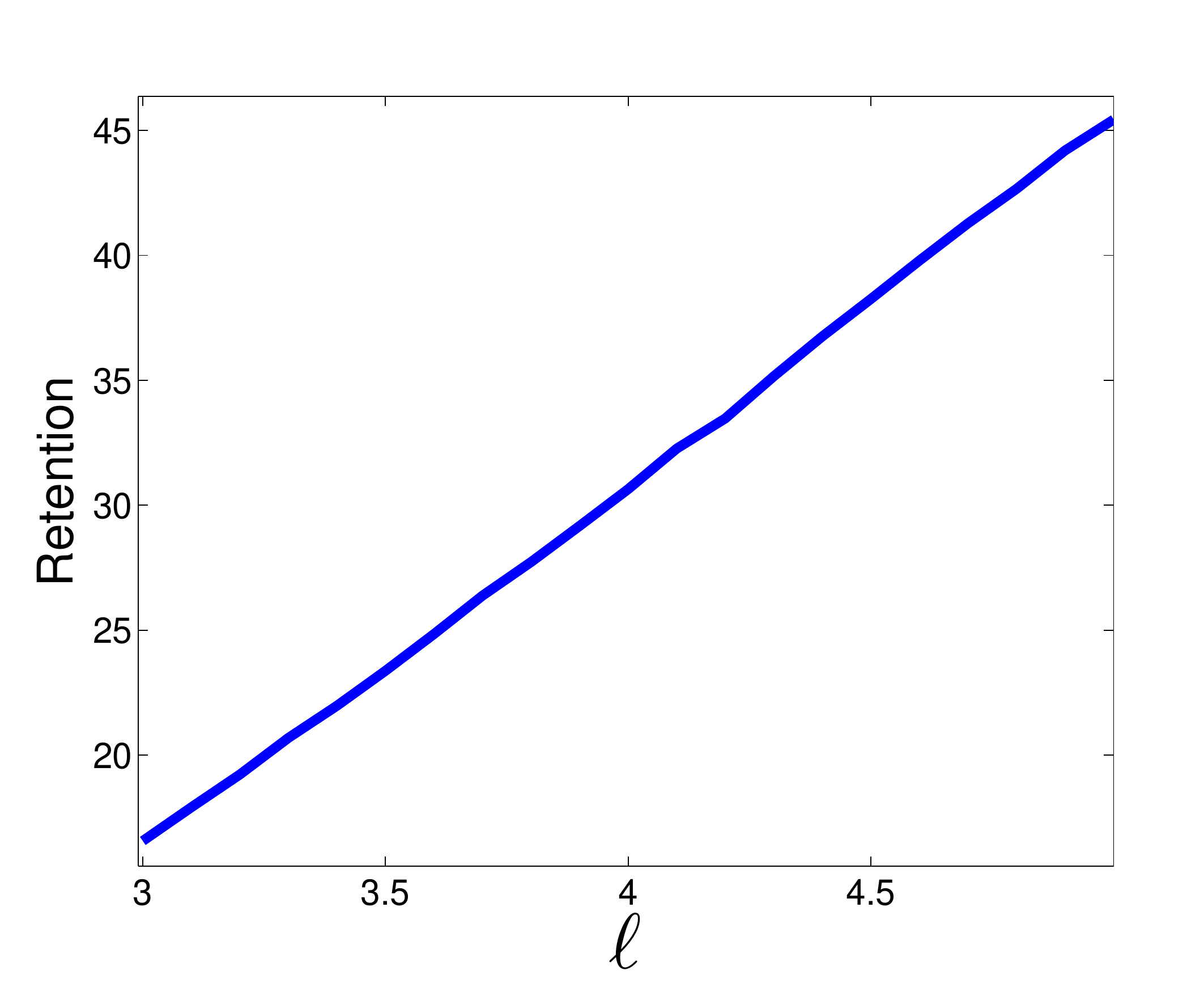}
}
\subfigure[\textbf{usps}]{
  \includegraphics[width=0.22\columnwidth,clip=true,trim=0.1in 0.05in 0.75in  0.4in]{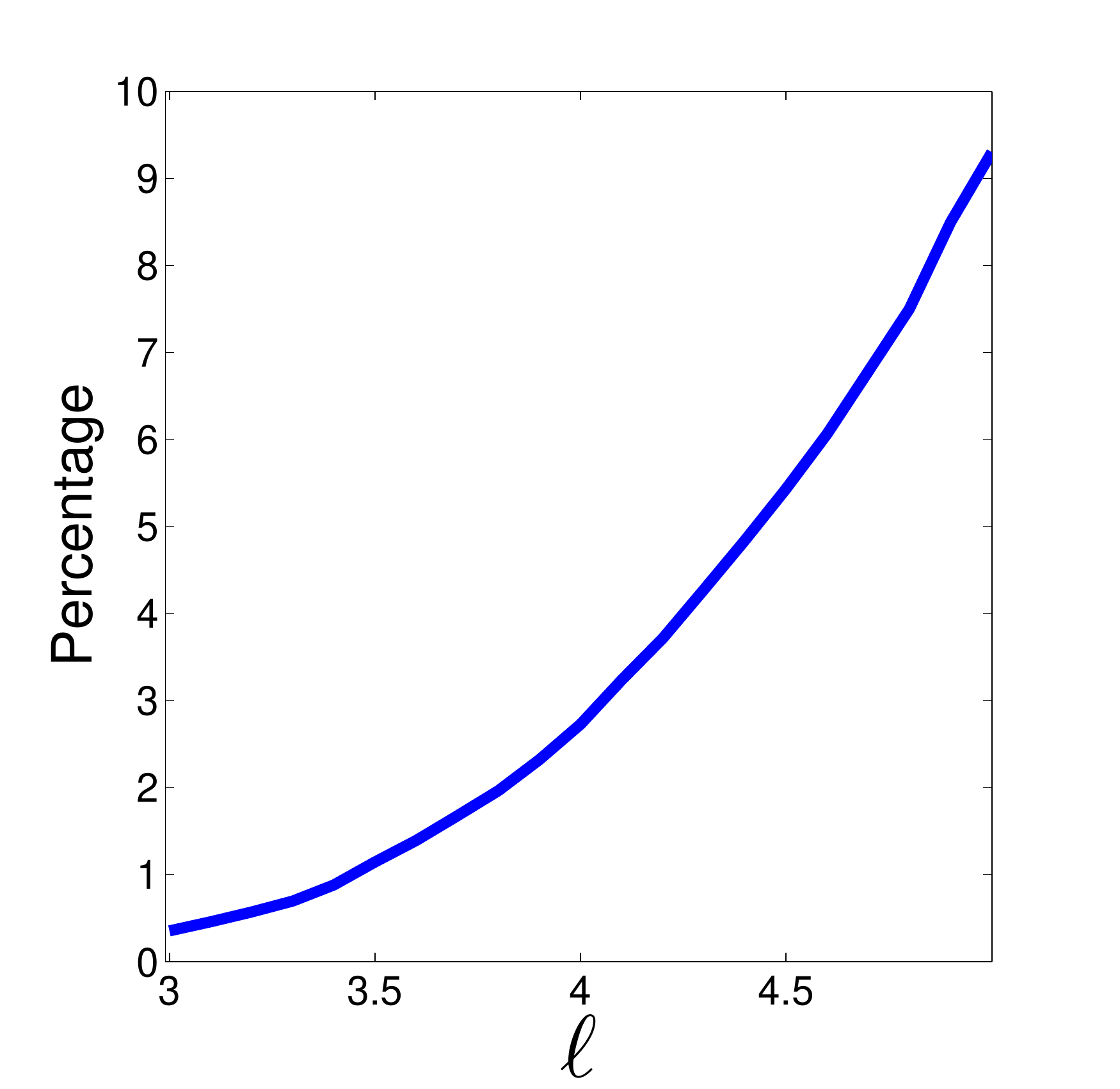}
  %\label{fig1:subfig1}
}
\subfigure[\textbf{yale}]{
  \includegraphics[width=0.22\columnwidth,clip=true,trim=0.1in 0.05in 0.75in  0.4in]{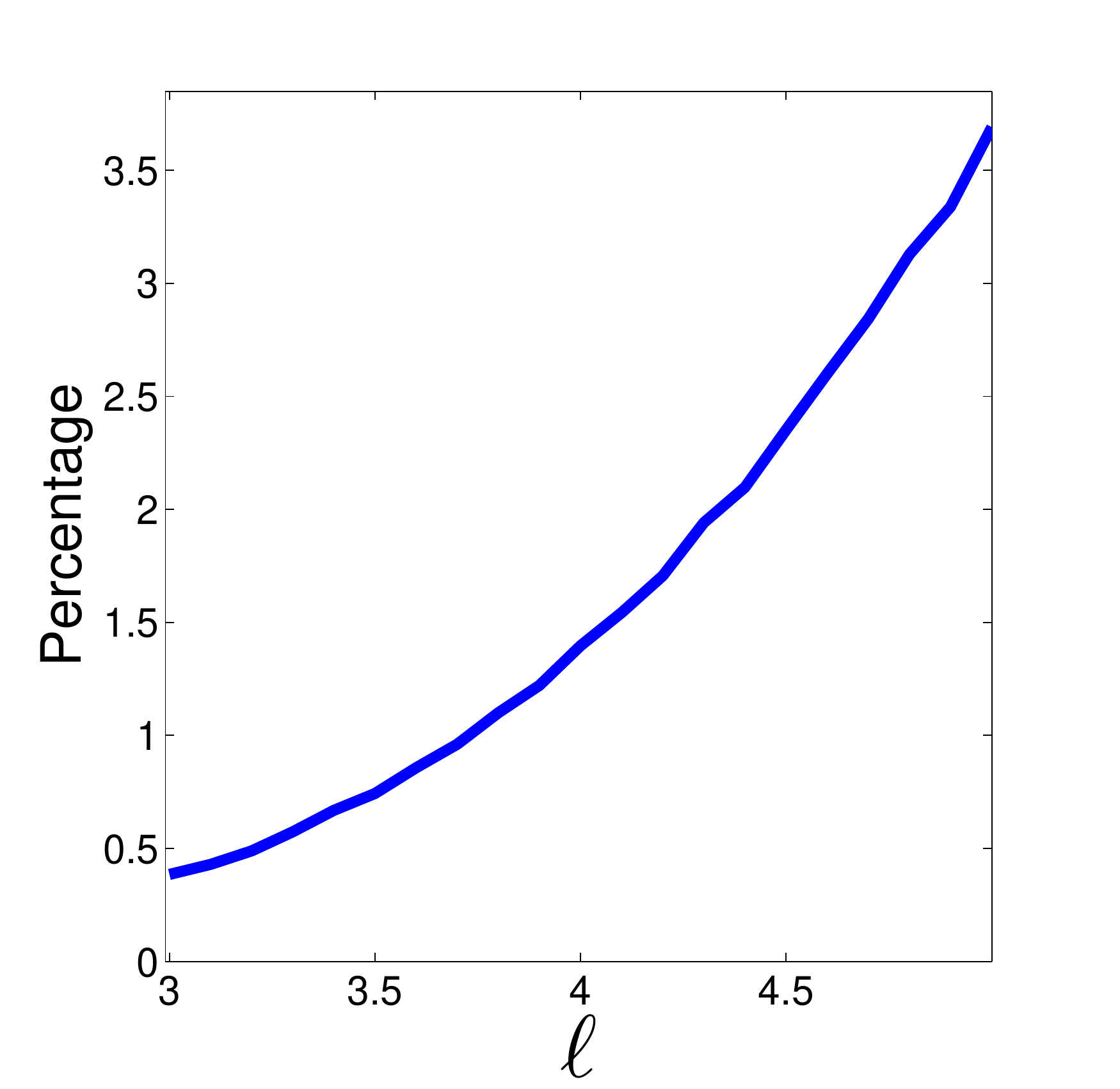}
  %\label{fig1:subfig1}
}
\caption{Percentage of data retained. \label{retained}}
\end{figure}

\begin{figure}[ht!]
\centering
\subfigure[Accuracy]{
  \includegraphics[scale=0.212,clip=true,trim=0.1in 0.05in 0.75in  0.4in]{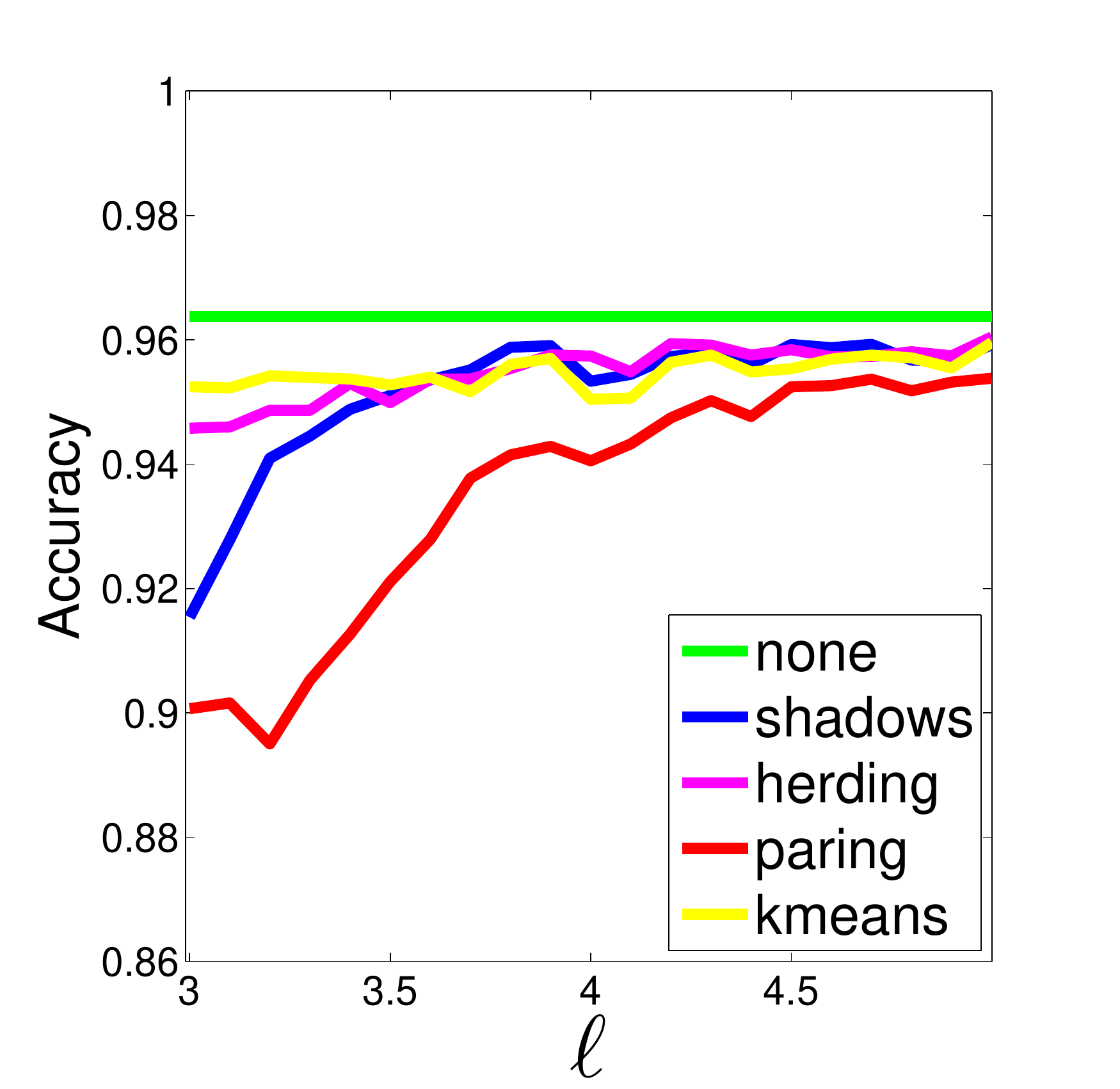}
  %\label{fig1:subfig1}
}
\subfigure[Training speedup]{
  \includegraphics[scale=0.212,clip=true,trim=0.1in 0.05in 0.75in  0.4in]{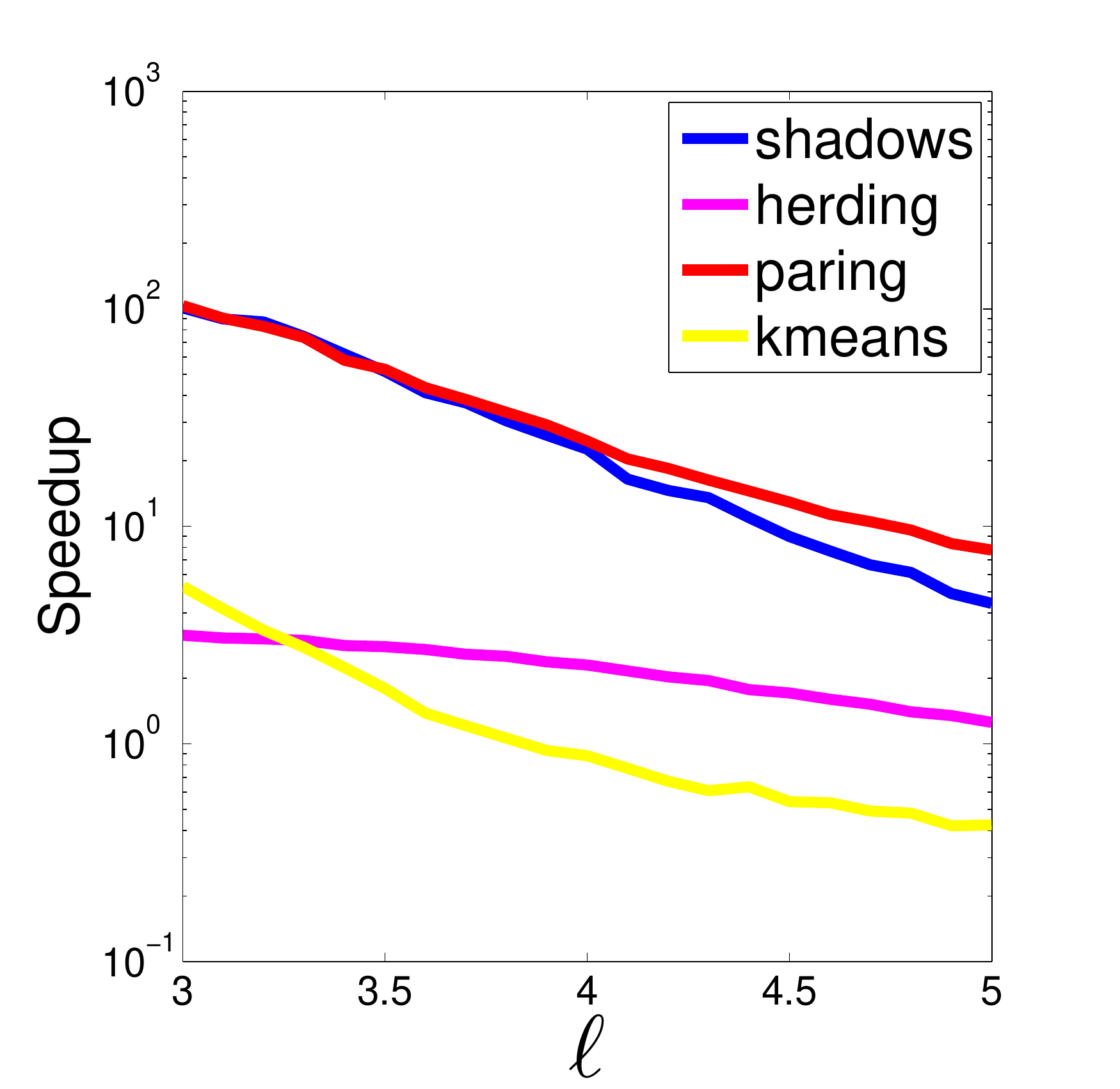}
}
\subfigure[Testing speedup]{
  \includegraphics[scale=0.212,clip=true,trim=0.1in 0.05in 0.75in  0.4in]{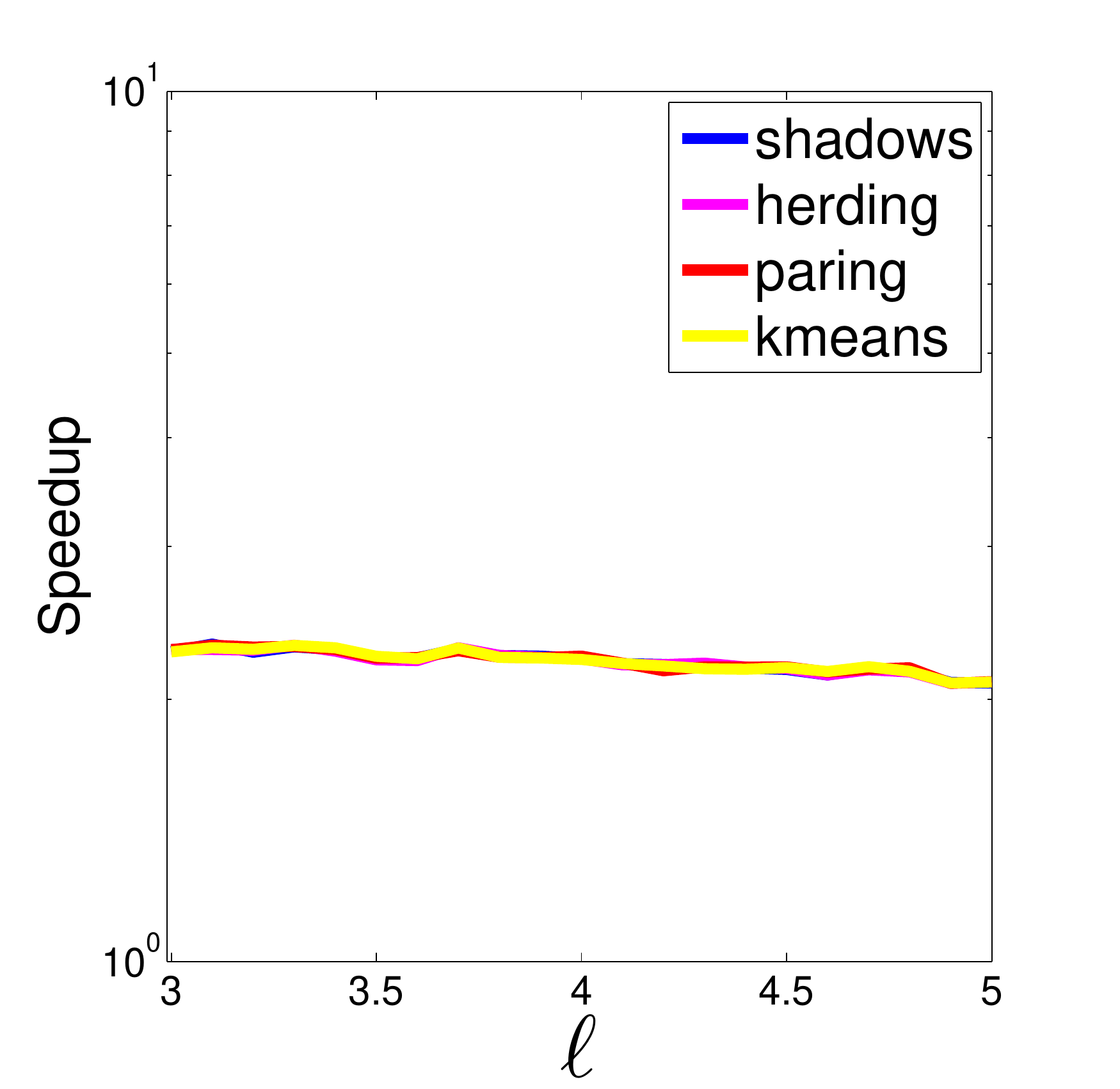}
  %\label{fig1:subfig1}
}
\subfigure[Total speedup]{
  \includegraphics[scale=0.212,clip=true,trim=0.1in 0.05in 0.75in  0.4in]{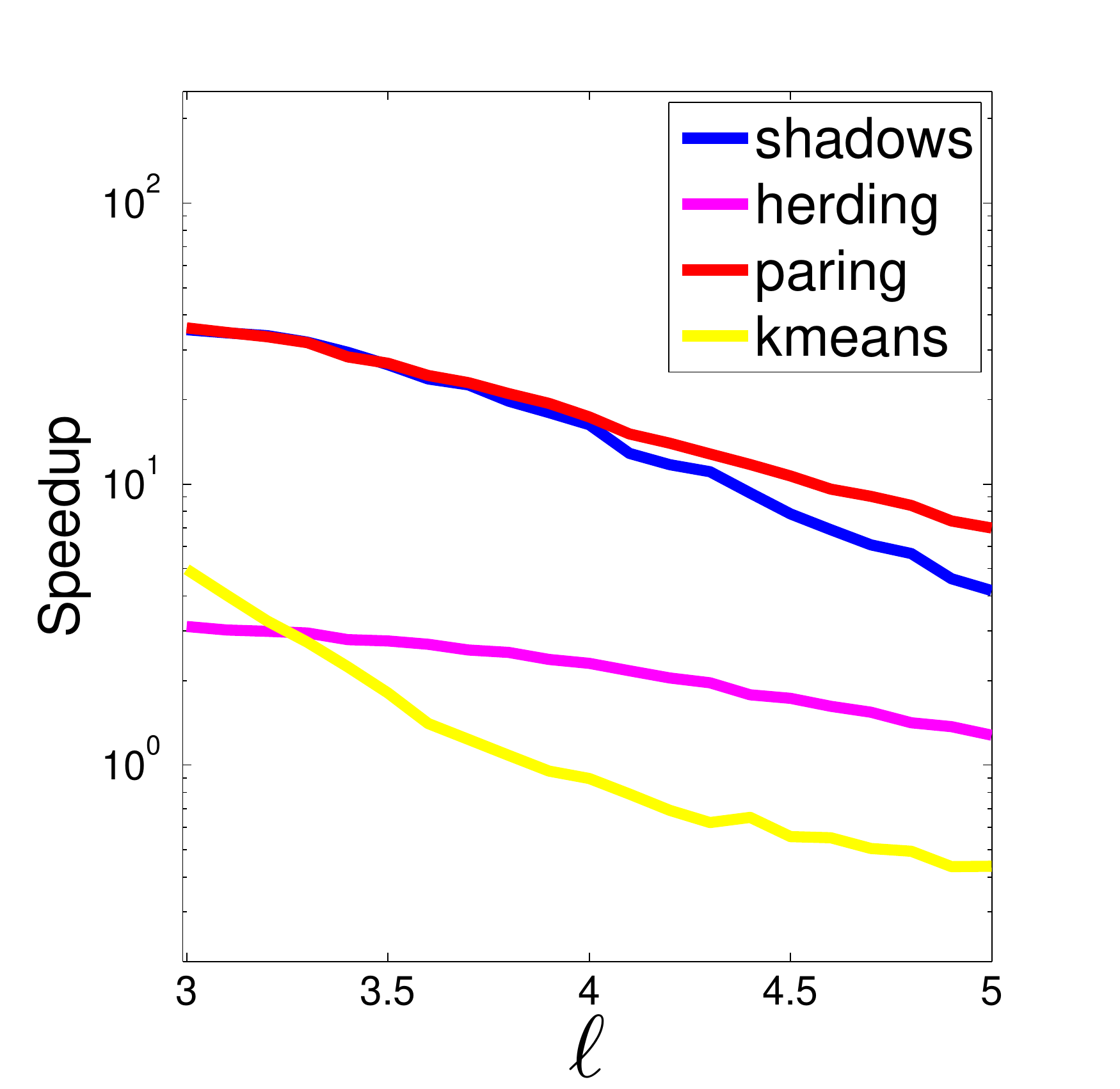}
  %\label{fig1:subfig1}
}
\caption{Classification comparisons w/RSDEs for \textbf{usps} as $\ell$ is
  varied ($\ntrain = 8,368$).}
\label{fig_center_usps}
\end{figure}

\begin{figure}[ht!]
\centering
\subfigure[Accuracy]{
  \includegraphics[scale=0.212,clip=true,trim=0.1in 0.05in 0.75in  0.4in]{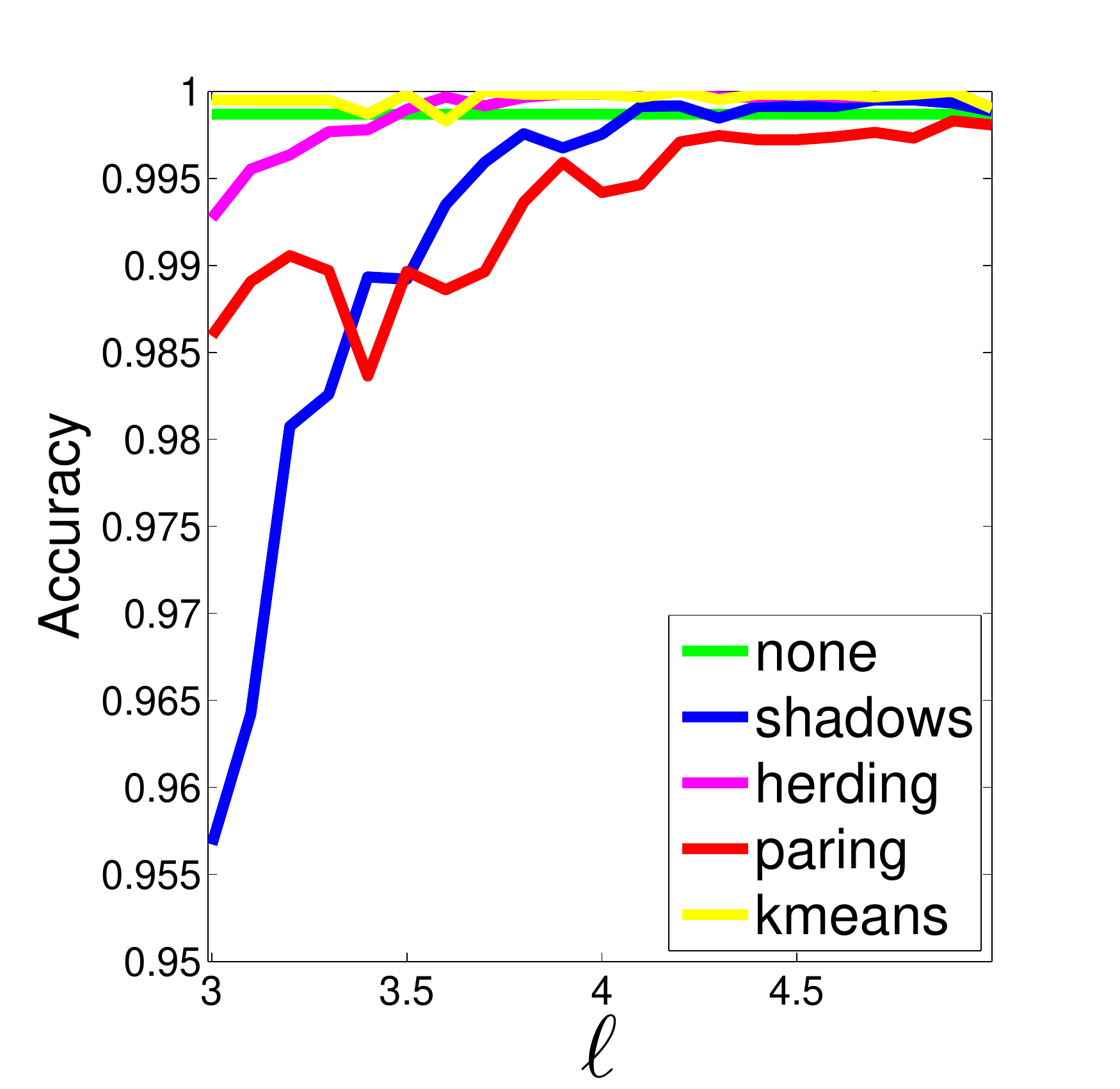}
  %\label{fig1:subfig1}
}
\subfigure[Training speedup]{
  \includegraphics[scale=0.212,clip=true,trim=0.1in 0.05in 0.75in  0.4in]{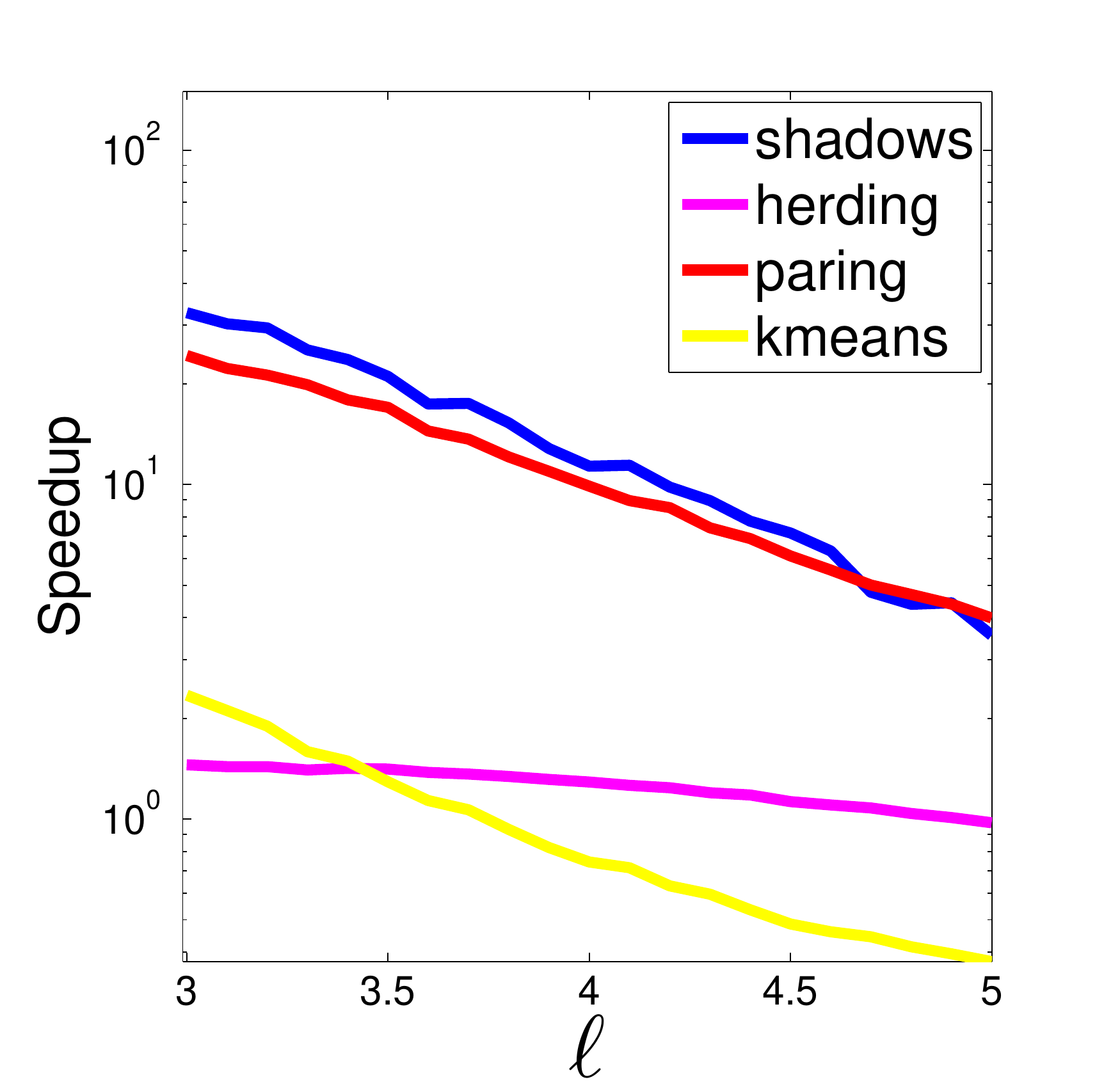}
}
\subfigure[Testing speedup]{
  \includegraphics[scale=0.212,clip=true,trim=0.1in 0.05in 0.75in  0.4in]{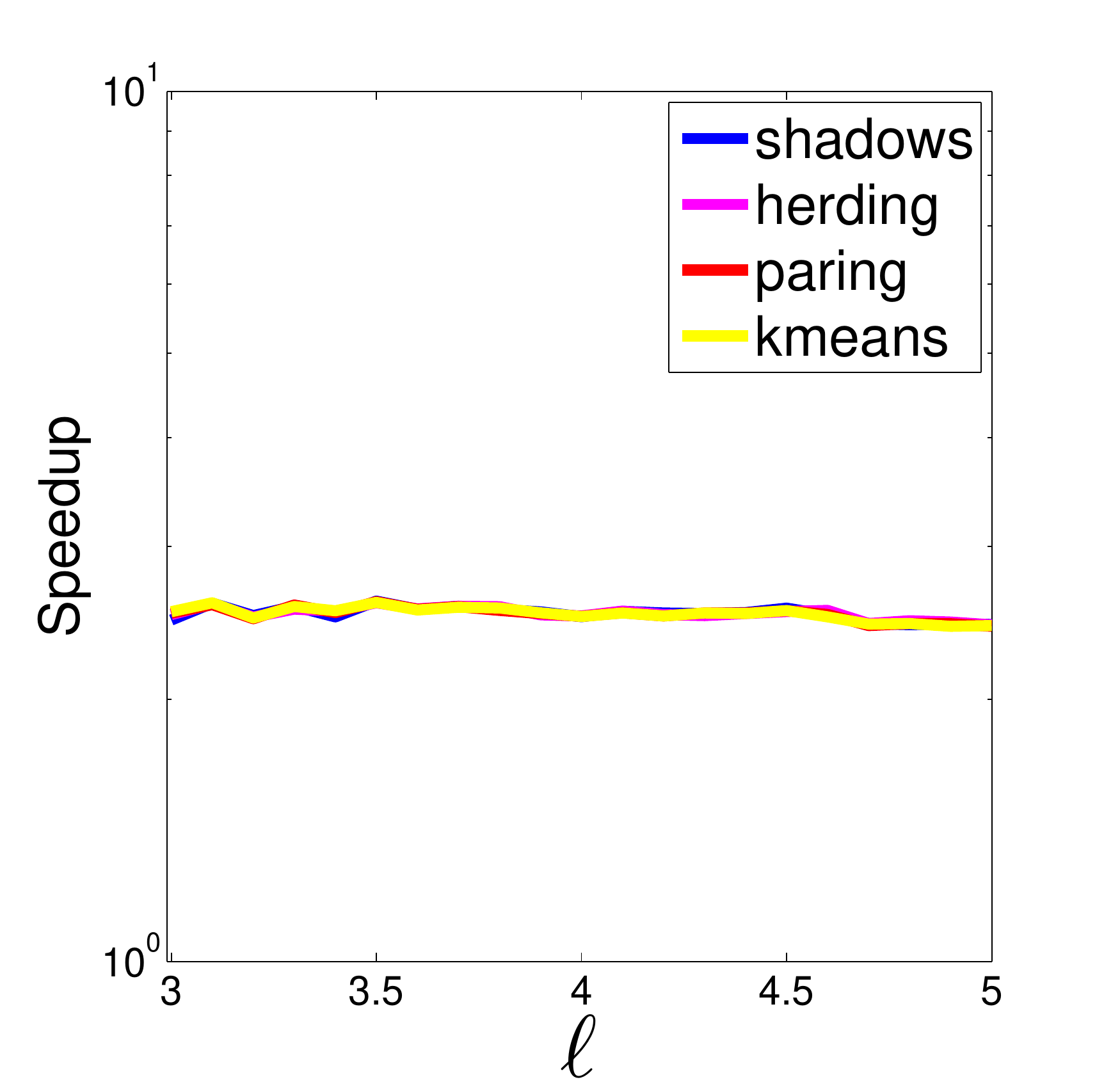}
  %\label{fig1:subfig1}
}
\subfigure[Total speedup]{
  \includegraphics[scale=0.212,clip=true,trim=0.1in 0.05in 0.75in  0.4in]{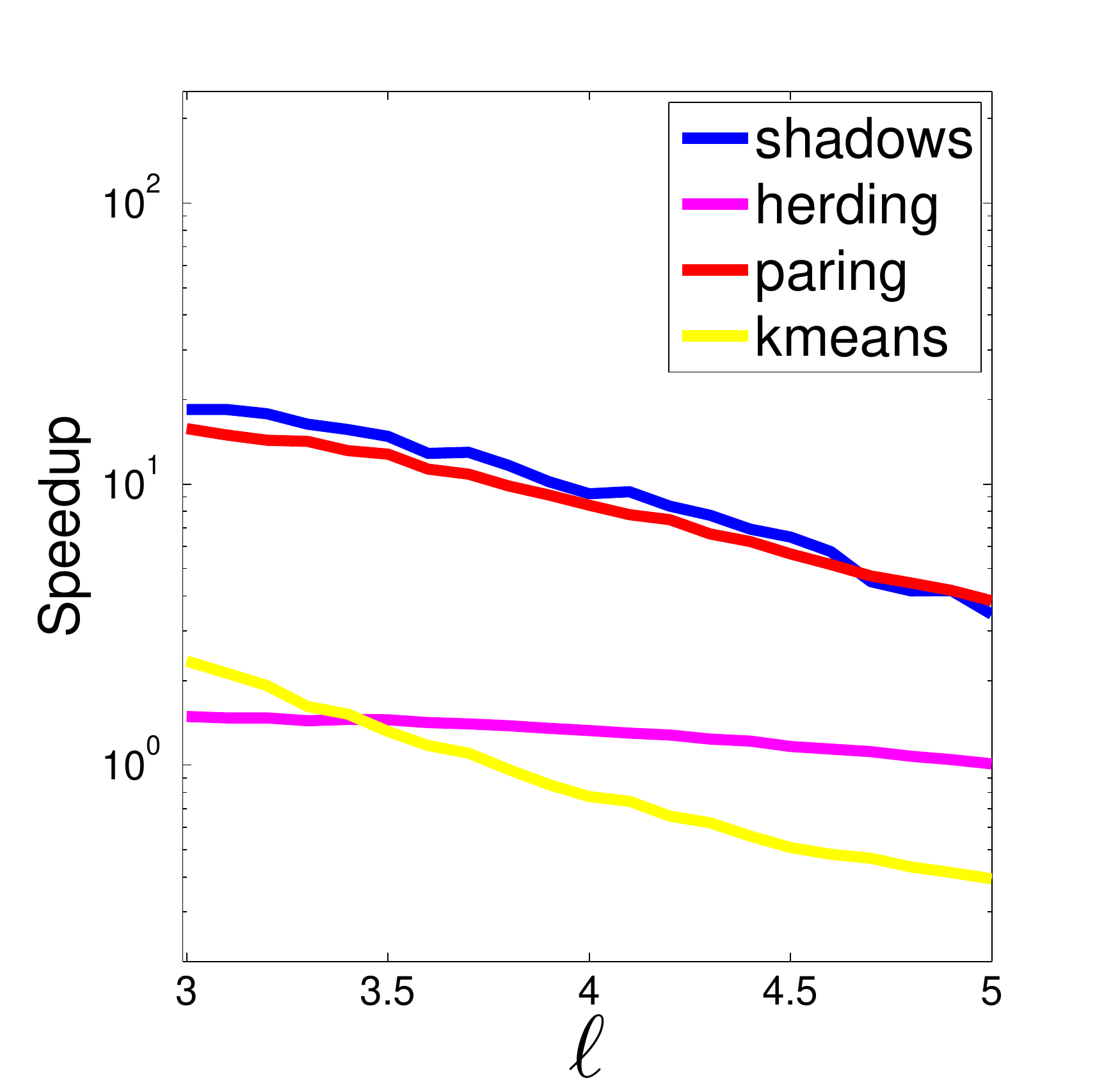}
  %\label{fig1:subfig1}
}
\caption{Classification comparisons w/RSDEs for \textbf{yale} as $\ell$ is
  varied ($\ntrain = 5,191$).}
\label{fig_center_yale}
\end{figure}

\section{Conclusion}
This paper presented (1) a reduced set KPCA algorithm for speeding up KPCA
given a reduce set density estimate of the training data, and (2) a simple,
efficient, single-pass algorithm for generating a suitable RSDE, %which we
called the shadow density estimate (ShDE), which relies on a user-selected
parameter $\ell$.  The spectral decomposition error was shown to be bounded
and directly related to the bound of the empirical error of the ShDE.  Through
ShDE+RSKPCA, significant reductions in both training and evaluation time are
achieved with minimal performance loss for large, redundant datasets.
Competitive overall speedups and performance were achieved versus Nystr\"{o}m methods.


\begin{thebibliography}{10}

\bibitem{Achlioptas:2002}
D.~Achlioptas, F.~McSherry, and B.~Sch{\"o}lkopf.
\newblock Sampling techniques for kernel methods.
\newblock In {\em NIPS}.

\bibitem{Arif:2011}
O.~Arif and P.A. Vela.
\newblock Kernel map compression for speeding the execution of kernel-based
  methods.
\newblock {\em IEEE Transactions on Neural Networks}, 22(6):870--879, 2011.

\bibitem{Belkin}
M.~Belkin and P.~Niyogi.
\newblock Laplacian {E}igenmaps for dimensionality reduction and data
  representation.
\newblock {\em Neural Computation}, 15(6):1373--1396, June 2003.

\bibitem{Bengio}
Y.~Bengio, P.~Vincent, and J-F. Paiement.
\newblock Learning eigenfunctions links {S}pectral {E}mbedding and {K}ernel
  {PCA}.
\newblock {\em Neural Computation}, 16(10):2197--2219, October 2004.

\bibitem{Chen:2010}
Y.~Chen, M.~Welling, and A.~Smola.
\newblock Super-samples from kernel herding.
\newblock In {\em Uncertainty in Artificial Intelligence}, 2010.

\bibitem{Coifman:2006}
R.R Coifman and S.~Lafon.
\newblock Diffusion maps.
\newblock {\em Applied and Computational Harmonic Analysis}, 21:5--30, 2006.

\bibitem{Drineas}
P.~Drineas and M.W. Mahoney.
\newblock On the {N}ystrom method for approximating a gram matrix for improved
  kernel-based learning.
\newblock {\em Journal of Machine Learning Research}, 6(10):2153--2175, October
  2005.

\bibitem{Freedman}
D.~Freedman and P.~Kisilev.
\newblock {KDE} {P}aring and a faster mean shift algorithm.
\newblock {\em {SIAM} Journal of Imaging Sciences}, 3(4):878--903, March 2010.

\bibitem{Ham}
J.~Ham, D.D. Lee, S.~Mika, and Bernhard Sch{\"o}lkopf.
\newblock A kernel view of the dimensionality reduction of manifolds.
\newblock In {\em ICML}, 2004.

\bibitem{Rosasco:2010}
L.~Rosasco, M.~Belkin, and E.D. Vito.
\newblock On learning with integral operators.
\newblock {\em Joural of Machine Learning Research}, 11:905--934, March 2010.

\bibitem{Scholkopf:1999}
B.~Sch{\"o}lkopf, S.~Mika, C.J.C. Burges, P.~Knirsch, K.~M{\"u}ller,
  G.~R{\"a}tsch, and A.J. Smola.
\newblock Input space versus feature space in kernel-based methods.
\newblock {\em IEEE Transactions on Neural Networks}, 10(5):1000--1017, 1999.

\bibitem{Scholkopf1998}
B.~Scholk{\"o}pf, A.~Smola, and R.~M{\"u}ller.
\newblock Nonlinear component analysis as a kernel eigenvalue problem.
\newblock {\em Neural Computation}, 10(5):1299--1319, July 1998.

\bibitem{Shawe-Taylor:2004}
J.~Shawe-Taylor and N.~Cristianini.
\newblock {\em Kernel Methods for Pattern Analysis}.
\newblock Cambridge University Press, New York, NY, USA, 2004.

\bibitem{Smola}
A.J. Smola, A.~Gretton, L.~Song, and B.~Scholkopf.
\newblock A {H}ilbert space embedding for distributions.
\newblock {\em Algorithmic Learning Theory}, 2007.

\bibitem{Tipping:2000}
M.~E. Tipping.
\newblock Sparse kernel principal component analysis.
\newblock In {\em NIPS}, pages 633--639, 2000.

\bibitem{Wang}
X.~Wang, P.~Tino, M.A. Fardal, S.~Raychaudhury, and A.~Babul.
\newblock Fast parzen window density estimator.
\newblock In {\em International Joint Conference on Neural Networks}, pages
  3267--3274, 2009.

\bibitem{Wasserman}
L.~Wasserman.
\newblock {\em All of Statistics}.
\newblock Springer, 2004.

\bibitem{Williams}
C.K.I. Williams and M.~Seeger.
\newblock The effect of the input density distribution on kernel-based
  classifiers.
\newblock In {\em ICML}, pages 1159--1166, 2000.

\bibitem{Zhang08}
K.~Zhang and J.T. Kwok.
\newblock Improved nystrom low rank approximation and error analysis.
\newblock In {\em ICML}, pages 1232--1239, 2008.

\bibitem{Zhang}
K.~Zhang and J.T. Kwok.
\newblock Density-weighted {N}ystrom method for computing large kernel
  eigensystems.
\newblock {\em Neural Computation}, 21(1):1299--1319, January 2010.

\bibitem{Zwald:2005}
L.~Zwald and G.~Blanchard.
\newblock On the convergence of eigenspaces in kernel principal component
  analysis.
\newblock In {\em NIPS}, 2005.

\end{thebibliography}
\end{document}